%% file: Paper_noProof.tex
\documentclass[runningheads]{llncs}
\usepackage{hyperref}
\usepackage{url}
\usepackage[T1]{fontenc}
\usepackage{graphicx}
\usepackage{threeparttable}
\input{math_commands.tex}

\usepackage{float}
\usepackage{wrapfig}
\usepackage{amsmath}
\usepackage{amssymb}
\usepackage{graphicx}
\usepackage{subcaption}
\usepackage{tabularx}
\usepackage{cases}
\usepackage{array}
\usepackage{epstopdf}
\usepackage{placeins}
\usepackage{pgfplots}
\usepackage{url}
\usepackage{tikz}
\usepackage{calc}
\usepackage{array}
\usepackage[misc]{ifsym}
\usepackage[linesnumbered,ruled,vlined]{algorithm2e}
\usetikzlibrary{positioning, arrows.meta,calc}

\newcommand{\val}{{\textrm{value}}}
\newcommand{\Val}{{\textrm{value}}}
\newcommand{\MILP}{{\textrm{MILP}}}
\newcommand{\LP}{{\textrm{LP}}}
\newcommand{\Improve}{\mathrm{Improve}}
\newcommand{\Utility}{\mathrm{SAS}}
\newcommand{\Sol}{\mathrm{Sol}}
\newcommand{\sol}{\mathrm{sol}}

\newcommand{\UB}{\mathrm{UB}}
\newcommand{\LB}{\mathrm{LB}}

\usepackage{amsmath, amssymb, amsfonts}

\newcommand{\ReLU}{\mathrm{ReLU}}

\newcommand{\toolname}{Hybrid MILP}

\begin{document}
	
	\title{Solution-aware vs global ReLU selection: \\
		partial MILP strikes back for DNN verification}
	
\author{Yuke Liao \Letter \inst{1}\orcidID{0009-0004-3763-686X} \and
	Blaise Genest\inst{2,3}\orcidID{0000-0002-5758-1876} \and
	Kuldeep Meel\inst{4}\orcidID{0000-0001-9423-5270}
	\and
	Shaan Aryaman\inst{5}\orcidID{0000-0001-7576-0766}}
\authorrunning{Y. Liao, B. Genest, K. Meel, S. Aryaman}
%
\institute{CNRS@CREATE, Singapore \email{yuke.liao@cnrsatcreate.sg} 
\and CNRS@CREATE \& IPAL, Singapore \email{blaise.genest@cnrsatcreate.sg}\and
	CNRS, IPAL, France \email{blaise.genest@cnrs.fr}\\
 \and
	University of Toronto,  Toronto, Canada \\\email{meel@cs.toronto.edu}
	\and
	NYU Courant Institute of Mathematical Sciences, New York, USA\\ \email{aryaman.shaan@gmail.com}}
\maketitle              
%


	\begin{abstract}
		Branch and Bound (BaB) is considered as the most efficient technique for DNN verification: it can propagate bounds over numerous branches, 
		to accurately approximate values a given neuron can take even in large DNNs, enabling formal verification of properties such as local robustness. Nevertheless, the number of branches grows {\em exponentially} with important variables, and there are complex instances for which the number of branches is too large to handle even using BaB. In these cases, providing more time to BaB is not efficient, as the number of branches treated is {\em linear} with the time-out. Such cases arise with verification-agnostic DNNs, non-local properties (e.g. global robustness, computing Lipschitz bound), etc. 
				
        To handle complex instances, we revisit a divide-and-conquer approach to break down the complexity: instead of few complex BaB calls, we rely on many small {\em partial} MILP calls. The crucial step is to select very few but very important ReLUs to treat using (costly) binary variables. The previous attempts were suboptimal in that respect. To select these important ReLU variables, we propose a novel {\em solution-aware} ReLU scoring ({\sf SAS}), as well as adapt the BaB-SR and BaB-FSB branching functions as {\em global} ReLU scoring ({\sf GS}) functions. 
		We compare them theoretically as well as experimentally, and {\sf SAS} is more efficient at selecting a set of variables to open using binary variables.
		Compared with previous attempts, SAS reduces the number of binary variables by around 6 times, while maintaining the same level of accuracy. Implemented in {\em Hybrid MILP}, calling first $\alpha,\beta$-CROWN with a short time-out to solve easier instances, and then partial MILP, produces a very accurate yet efficient verifier, reducing by up to $40\%$ the number of undecided instances to low levels ($8-15\%$), while keeping a reasonable runtime ($46s-417s$ on average per instance), 
		even for fairly large CNNs with 2 million parameters.

		
	\end{abstract}

\section{Introduction}

\input{introduction}
\input{relatedwork}

\section{Notations and Preliminaries}

In this paper, we will use lower case latin $a$ for scalars, bold $\boldsymbol{z}$ for vectors, 
capitalized bold $\boldsymbol{W}$ for matrices, similar to notations in \cite{crown}.
To simplify the notations, we restrict the presentation to feed-forward, 
fully connected ReLU Deep Neural Networks (DNN for short), where the ReLU function is $\ReLU : \mathbb{R} \rightarrow \mathbb{R}$ with
$\ReLU(x)=x$ for $x \geq 0$ and $\ReLU(x)=0$ for $x \leq 0$, which we extend componentwise on vectors.



An $\ell$-layer DNN is provided by $\ell$ weight matrices 
$\boldsymbol{W}^i \in \mathbb{R}^{d_i\times d_{i-1}}$
and $\ell$ bias vectors $\vb^i \in \mathbb{R}^{d_i}$, for $i=1, \ldots, \ell$.
We call $d_i$ the number of neurons of hidden layer $i \in \{1, \ldots, \ell-1\}$,
$d_0$ the input dimension, and $d_\ell$ the output dimension.

Given an input vector $\boldsymbol{z}^0 \in \mathbb{R}^{d_0}$, 
denoting $\hat{\boldsymbol{z}}^{0}={\boldsymbol{z}}^0$, we define inductively the value vectors $\boldsymbol{z}^i,\hat{\vz}^i$ at layer $1 \leq i \leq \ell$ with
\begin{align}
	\boldsymbol{z}^{i} = \boldsymbol{W}^i\cdot \hat{\boldsymbol{z}}^{i-1}+ \vb^i \qquad \, \qquad
	\hat{\boldsymbol{z}}^{i} = \ReLU({\boldsymbol{z}}^i).
\end{align} 

The vector $\hat{\boldsymbol{z}}$ is called post-activation values, 
$\boldsymbol{z}$ is called pre-activation values, 
and $\boldsymbol{z}^{i}_j$ is used to call the $j$-th neuron in the $i$-th layer. 
For $\boldsymbol{x}=\vz^0$ the (vector of) input, we denote by $f(\boldsymbol{x})=\vz^\ell$ the output. Finally, pre- and post-activation neurons are called \emph{nodes}, and when we refer to a specific node/neuron, we use $a,b,c,d,n$ to denote them, and $W_{a,b} \in \mathbb{R}$ to denote the weight from neuron $a$ to $b$. Similarly, for input $\boldsymbol{x}$, we denote by $\val_{\boldsymbol{x}}(a)$ the value of neuron $a$ when the input is $\boldsymbol{x}$. 

\medskip

Concerning the verification problem, we focus on the well studied local-robustness question. Local robustness asks to determine whether the output of a neural network will be affected under small perturbations to the input. 
Formally, for an input $\vx$ perturbed by $\varepsilon >0$ under distance $d$, then the DNN is locally $\varepsilon$-robust in $\vx$ whenever:

	\begin{align}
	\forall \boldsymbol{x'} \text{ s.t. } d(\vx,\vx')\leq \varepsilon, \text{ we have }  
	{argmax (f(\boldsymbol{x'})) = argmax(f(\boldsymbol{x}))}
\end{align}

\input{valabstraction}

\input{Comparison}

\input{formula}

\input{experiments}
\section{Conclusion}
In this paper, we developed a novel solution-aware scoring ({\sf SAS}) function to select few ReLU nodes to consider with binary variables to compute accurately bounds in DNNs. 
The solution awareness allows SAS to compute an accurate score for each ReLU, which enables partial MILP to be very efficient, necessitating $\approx6$x less binary variables than previous proposals \cite{DivideAndSlide} for the same accuracy, and $\approx2$x less than {\sf GS} scoring adapted from FSB \cite{FSB}. As the worst-case complexity is exponential in the number of binary variables, this has large implication in terms of scalability to larger DNNs, making it possible to verify accurately quite large DNNs such as CNN-B-Adv with 2M parameters. 

While $\alpha,\beta$-CROWN is known to be extremely efficient to solve easier verification instances, we exhibit many cases (complex instances) where its worst-case exponential complexity in the number of ReLUs is tangible, with unfavorable scaling (Table \ref{table_beta}). Resorting to Hybrid MILP, a divide-and-conquer approach \cite{DivideAndSlide}, revisited thanks to the very efficient {\sf SAS}, revealed to be a much better trade-off than augmenting $\alpha,\beta$-CROWN time-outs, with $8\%$ to $40\%$ less undecided images at ISO runtime. Currently, for hard instances, there is no alternative to partial MILP, other methods being $>10$ times slower.

This opens up interesting future research directions, to verify global \cite{lipshitz}, \cite{sensing} rather than local (robustness) properties, which need very accurate methodology
and give rise to hard instances as the range of each neuron is no more local to a narrow neighborhood (and thus almost all ReLUs are unstable, with both modes possible).



\smallskip
{\bf Acknowledgement:} 
This research was conducted as part of the DesCartes program and 
was supported by the National Research Foundation, Prime Minister's Office, Singapore, 
under the Campus for Research Excellence and Technological Enterprise
(CREATE) program, and partially supported by ANR-23-PEIA-0006 SAIF.

\bibliography{references}
\bibliographystyle{plain}

\newpage

\bigskip

\appendix

\input{proofsb}

\end{document}

%% file: math_commands.tex
\usepackage{amsmath,amsfonts,bm}

\def\eqref#1{equation~\ref{#1}}

\def\1{\bm{1}}

\def\vb{{\bm{b}}}

\def\vx{{\bm{x}}}
\def\vy{{\bm{y}}}
\def\vz{{\bm{z}}}

\DeclareMathAlphabet{\mathsfit}{\encodingdefault}{\sfdefault}{m}{sl}
\SetMathAlphabet{\mathsfit}{bold}{\encodingdefault}{\sfdefault}{bx}{n}



%% file: introduction.tex
Deep neural networks (DNNs for short) have demonstrated remarkable capabilities, achieving human-like or even superior performance across a wide range of tasks. However, their robustness is often compromised by their susceptibility to input perturbations \cite{szegedy}. This vulnerability has catalyzed the verification community to develop various methodologies, each presenting a unique balance between completeness and computational efficiency \cite{Marabou,Reluplex,deeppoly}. This surge in innovation has also led to the inception of competitions such as VNNComp \cite{VNNcomp}, which aim to systematically evaluate the performance of neural network verification tools. While the verification engines are generic, the benchmarks usually focus on local robustness, i.e. given a DNN, an image and a small neighbourhood around this image, is it the case that all the images in the neighbourhood are classified in the same way. 
For the past 5 years, VNNcomp has focused on rather easy instances, that can be solved within tens of seconds (the typical hard time-out is 300s). For this reason, DNN verifiers in the past years have mainly focused on optimizing for such easy instances. Among them, NNenum \cite{nnenum}, Marabou \cite{Marabou,Marabou2}, and PyRAT 
\cite{pyrat}, respectively 4th, 3rd and 2{nd} of the last VNNcomp'24 \cite{VNNcomp24}
and 5th, 2{nd} and 3rd  of the VNNcomp'23 \cite{VNNcomp23}; MnBAB \cite{ferrari2022complete}, 2{nd} in VNNcomp'22 \cite{VNNcomp22}, built upon ERAN \cite{deeppoly} and PRIMA \cite{prima}; and importantly, $\alpha,\beta$-CROWN \cite{crown,xu2020fast}, the winner of the last 4 VNNcomp, benefiting from branch-and-bound based methodology \cite{cutting,BaB}.
We will thus compare mainly with $\alpha,\beta$-CROWN experiments as gold standard in the following\footnote{GCP-CROWN \cite{cutting} is slightly more accurate than $\alpha,\beta$-CROWN on the DNNs we tested, but necessitates IBM CPLEX solver, which is not available to us}.

\begin{table}[t!]
	\centering
	\begin{tabular}{||l|c|c|c||c|c|c||}
		\hline
		Network & nbr & Accur.  & Upper & $\alpha,\beta$-CROWN& $\alpha,\beta$-CROWN & $\alpha,\beta$-CROWN \\ 
		Perturbation & activ. &  & Bound & TO=10s & TO=30s & TO=2000s\\ \hline
		MNIST 5$\times$100 & 500 & 99\% & 90\% & 33\% & 35\% & 40\%   \\
		$\epsilon = 0.026$ & ReLU & &  & 6.9s &  18.9s &  1026s  \\  \hline
		MNIST 5$\times$200 & 1000 & 99\%  & 96\%  & 46\%  & 49\%  & 50\%   \\ 
		$\epsilon = 0.015$ & ReLU & &  & 6.5s &  16.6s &  930s  \\  \hline
		MNIST 8$\times$100 & 800 & 97\%  & 86\%  & 23\%  & 28\%  & 28\%   \\
		$\epsilon = 0.026$ & ReLU &  &  & 7.2s &  20.1s &  930s  \\  \hline
		MNIST 8$\times$200 & 1600 & 97\%  & 91\%  & 35\%  & 36\%  & 37\%   \\ 
		$\epsilon = 0.015$ & ReLU & &  & 6.8s &  18.2s &  1083s  \\  \hline
		MNIST 6$\times$500 & 3000 & 100\%  & 94\%  & 41\%  & 43\%  & 44\%   \\ 
		$\epsilon = 0.035$ & ReLU & &  & 6.4s &  16.4s &  1003s  \\  \hline
		CIFAR CNN-B--Adv &  16634  &   78\%  & 62\%  &  34\% & 40\%  & 42\%   \\
		$\epsilon = 2/255$& ReLU &  &  & 4.3s & 8.7s & 373s  \\ \hline \hline
		CIFAR ResNet &   107496  & 29\%  & 25\%  & 25\%  & 25\%  & 25\%   \\
		$\epsilon = 8/255$ & ReLU &  &  & 2s & 2s & 2s  \\ \hline
	\end{tabular}
		\caption{Accuracy of DNN (class predicted vs ground truth), upper bound on robustness (robustness attacks found on remaining images), and $\%$ of images verified by $\alpha,\beta$-CROWN with different time-outs (TO) on 7 DNNs, and average runtime per image. The 6 first DNNs are complex instances. The last DNN (ResNet) is an easy instance (trained using Wong to be easy to verify, but with a very low accuracy level), provided for reference.
	}
\label{table_beta}
\end{table}

$\alpha,\beta$-CROWN, as well as BaBSR \cite{BaB} and MN-BaB \cite{ferrari2022complete},
rely on Branch and Bound technique (BaB), which call BaB once per output neuron (few calls). In the worst case, this involves considering all possible ReLU configurations, though branch and bound typically circumvents most possibilities. For easy instances, BaB is highly efficient as all branches can be pruned early. However, BaB methods hit a complexity barrier when verifying more complex instances, due to an overwhelming number of branches (exponential in the height of branches that cannot be pruned as they need too many variables to branch over). This can be clearly witnessed on the verification-agnostic \cite{SDPFI} DNNs of Table \ref{table_beta} (6 first DNNs), where vastly enlarging the time-out only enables to verify few more \% of images, leaving a large proportion ($20\%-50\%$) of images undecided despite the large runtime. As argued in \cite{SDPFI}, there are many situations (workflow, no access to the dataset...) where using specific trainers to learn easy to verify DNN is simply not possible, leading to  {\em verification-agnostic} networks, and such cases should be treated as well as DNNs specifically trained to be easy to verify, e.g. using \cite{TrainingforVerification}. Verification-agnostic are the simplest instances to demonstrate the scaling behavior of BaB on complex instances using standard local robustness implementations. Other complex instances include solving non-local properties, e.g. global robustness computed through Lipschitz bound \cite{lipshitz}, etc. The bottom line is that one cannot expect to have only easy instances to verify. It is important to notice that the number of activation functions of the DNN is a poor indicator of the hardness of the instance, e.g. $5 \times 100$ with 500 ReLUs is far more complex to certify ($50\%$ undecided images) than 100 times bigger ResNet ($0\%$ undecided images), see Table \ref{table_beta}.

\begin{table}[b!]
	\centering
	\begin{tabular}{||l|c|c||c|c|c||}
		\hline
		Network &  Accuracy & Upper  & Marabou 2.0 & NNenum &  Full MILP  \\ \hline
		MNIST 5$\times$100 & 99\% & 90\% & 28\% & $49\%$ & 40 \%    \\
		$\epsilon = 0.026$ & &  &6200s &  4995s & 6100s
		  \\  \hline
	\end{tabular}
	\caption{Result of non-BaB methods on the hard $5 \times 100$ with TO = 10 000s. 
		Only NNenum verifies more instances (9\% out of 50\% undecided images) than $\alpha,\beta$-CROWN (40\%), at the cost of a much larger runtime (4995s vs 1026s).
		}
\label{table_complete}
\end{table}

Other standard non-BaB methods such as Marabou, NNenum or a Full MILP encoding, show similar poor performance on such complex instances as well, even with a large 10 000s Time-out: Table \ref{table_complete} reveals that only NNenum succeeds to verify images not verified by $\alpha,\beta$-CROWN, limited to 9\% more images out of the 50\% undecided images on $5 \times 100$, and with a very large runtime of almost 5000s per image. It appeared pointless to test these verifiers on larger networks.

Our main contributions address the challenges to verify {\em complex} instances efficiently, as current methods are not appropriate to verify such instances:
\begin{enumerate}
	
	\item We revisit the idea from \cite{DivideAndSlide} to consider small calls to a partial MILP (pMILP) solver, i.e. with few binary variables encoding few ReLU functions exactly (other being encoded with more efficient but less accurate linear variables), to compute bounds for each neuron inductively, hence many ($O(n)$, the number of neuron) small calls, with a complexity exponential only in the few binary variables (Section 4). Compared to the few (one per output neuron) complex call to BaB, each with a worst case complexity exponential in the number of neurons of the DNN (which is far from the actual complexity thanks to pruning branches in BaB - but which can be too large as we shown in the 6 first DNNs of Table \ref{table_beta}). Two questions arise: how to select few very important ReLUs?, and is computing the bounds of intermediate neurons a good trade-off compared with the theoretical loss of accuracy due to selecting only some binary ReLUs? Answer to these questions were not looking very promising judging by previous attempt \cite{DivideAndSlide}, which was using a simple selection heuristic of nodes in the previous layer only.

	\item On the first question, we adapted from BaB-SR \cite{BaB} and FSB \cite{FSB}, which choose branching nodes for BaB, {\em global scoring ({\sf GS})} functions to choose important ReLUs (Section 5). These revealed to be much more accurate than the simple heuristic in 
	\cite{DivideAndSlide}. However, we also uncover that the {\em improvement function} that {\sf GS} tries to approximate depends heavily upon the mode of the ReLU function, and as this mode is unavailable to {\sf GS}, there are many cases in which {\sf GS} is far from the improvement (with both under or over-approximation of the improvement function). 

	\item We thus designed a {\em novel solution-aware scoring ({\sf SAS})}, which uses the solution of a unique LP call, that provides the mode to consider (Section 6). Theoretically, we show that {\sf SAS} is always an over-approximation of the improvement (Proposition \ref{prop2}), which implies that a small {\sf SAS} value implies that the ReLU is unnecessary. Experimentally, we further show that {\sf SAS} is very close to the actual improvement, closer than {\sf GS}, and that overall, the accuracy from SAS is significantly better than {\sf GS}.  Compared with the heuristic in \cite{DivideAndSlide},  {\sf SAS} is much more efficient ($\approx 6$ times less binary variables for same accuracy (Fig. \ref{fig_table3})). 
	
	\item Compared with many calls using full MILP, where all the ReLUs are encoded as binary variables, {\sf SAS} (and {\sf GS}) encode only a subset of ReLUs as binary and others as linear variables. The model in full MILP is fully accurate, while {\sf SAS} (and {\sf GS})
	are abstractions, and thus much faster to solve. 
	While full MILP is thus {\em asymptotically} more accurate than {\sf SAS} and {\sf GS}, {\em experimentally}, every reasonable time-out leads to much better practical accuracy of {\sf SAS} (and {\sf GS}) (see Fig. \ref{fig555}).

	
	\item For the second question, we propose a new verifier, called {\em Hybrid MILP}, invoking first 	$\alpha,\beta$-CROWN with short time-out to settle the easy instances. On those ({\em hard}) instances which are neither certified nor falsified, we call pMILP with few neurons encoded as binary variables. Experimental evaluation reveals that Hybrid MILP achieves a beneficial balance between accuracy and completeness compared to prevailing methods. It reduces the proportion of undecided inputs from $20-58\%$ ($\alpha,\beta$-CROWN with 2000s TO) to $8-15\%$, while taking a reasonable average time per instance ($46-420$s), Table \ref{table_hybrid}. It scales to fairly large networks such as CIFAR-10 CNN-B-Adv \cite{SDPFI}, with more than 2 million parameters.
\end{enumerate}

Limitation: We consider DNNs employing the standard ReLU activation function, though our findings can be extended to other activation functions, following similar extention by \cite{DivideAndSlide}, with updated MILP models e.g. for maxpool.

%% file: relatedwork.tex
\subsection{Related Work} 

We compare Hybrid MILP with major verification tools for DNNs to clarify our methodology and its distinction from the existing state-of-the-art. It scales while preserving good accuracy, through targeting a limited number of binary variables, stricking a good balance between exact encoding of a DNN using MILP~\cite{MILP} (too slow) and LP relaxation (too inaccurate). MIPplanet~\cite{MIPplanet} opts for a different selection of binary variables, and execute one large MILP encoding instead of Hybrid MILP's many small encodings, which significantly reduce the number of binary variables necessary for each encoding. In \cite{DivideAndSlide}, small encodings are also considered, however with a straightforward choice of binary nodes based on the weight of outgoing edges, which need much more {binary variables} (thus runtime) to reach the same accuracy.

Compared with \cite{atva}, which uses pMILP in an abstraction refinement loop, 
they iteratively call pMILP to obtain bounds for the same output neuron, opening more and more ReLUs. This scales only to limited size DNN (500 neurons), because of the fact that many ReLUs need to be open (and then Gurobi takes a lot of time) and the iterative nature which cannot be parallelized, unlike our method which scales up to 20.000 neurons.


The fact that pure BaB is not that efficient for e.g. verification-agnostic (even very small) DNNs has been witnessed before \cite{MILP2}. The workaround, e.g. in {\em refined} $\alpha,\beta$-CROWN, was to precompute very accurate bounds for the first few neurons of the DNN using a complete full MILP encoding, and then rely on a BaB call from that refined bounds (more complex calls to full MILP and BaB). Non-surprisingly, this very slow technique does not scale but to small DNNs (max 2000 ReLU activation functions). Hybrid MILP on the other hand relies only on small calls: it is much more efficient on small DNNs, and it can scale to larger DNNs as well: we demonstrated strong performance with at least one order of magnitude larger networks (CNN-B-Adv).


Last, ERAN-DeepPoly \cite{deeppoly} computes bounds on values very quickly, by abstracting the weight of every node using two functions: an upper function and a lower function. While the upper function is fixed, the lower function offers two choices.
It relates to the LP encoding through Proposition \ref{LP} \cite{alessandro}: the LP relaxation precisely matches the intersection of these two choices. Consequently, LP is more accurate (but slower) than DeepPoly, and Hybrid MILP is considerably more precise. Regarding PRIMA \cite{prima}, the approach involves explicitly maintaining dependencies between neurons.



Finally, methods such as Reluplex / Marabou \cite{Reluplex,Marabou}  abstract the network: they diverge significantly from those abstracting values such as PRIMA, $\alpha,\beta$-CROWN)\cite{prima,crown}, Hybrid MILP. These network-abstraction algorithms are designed to be {\em intrinsically complete} (rather than asymptotically complete as BaB), but this comes at the price of significant scalability challenges, and in practice they time-out on complex instances as shown in Table \ref{table_complete}.

%% file: valabstraction.tex
\section{Value Abstraction for DNN verification}

In this section, we describe different value (over-)abstractions on $\vz$ that are used by efficient algorithms to certify robustness around an input $\vx$. Over-abstractions of values include all values for $\vz$ in the neighbourhood of $\vx$, and thus a certificate for safety in the over-abstraction is a proof of safety for the original input $\vx$.

\subsection{The Box Abstractions}

\definecolor{applegreen}{rgb}{0.55, 0.71, 0.0}

The concept of value abstraction involves calculating upper and lower bounds for the values of certain neurons in a Deep Neural Network (DNN) when inputs fall within a specified range. This approach aims to assess the network's robustness without precisely computing the values for every input within that range.

Firstly, it's important to note that weighted sums represent a linear function, which can be explicitly expressed with relative ease. However, the ReLU (Rectified Linear Unit) function presents a challenge in terms of accurate representation. Although ReLU is a relatively straightforward piecewise linear function with two modes (one for $x<0$ and another for $x \geq 0$), it is not linear. The complexity arises when considering the compounded effects of the ReLU function across the various layers of a ReLU DNN. It's worth noting that representing $\ReLU(x)$ precisely is feasible when $x$ is "{\em stable}", meaning it's consistently positive or consistently negative, as there's only one linear mode involved in each scenario. Consequently, the primary challenge lies in addressing "{\em unstable}" neurons, where the linearity of the function does not hold consistently.

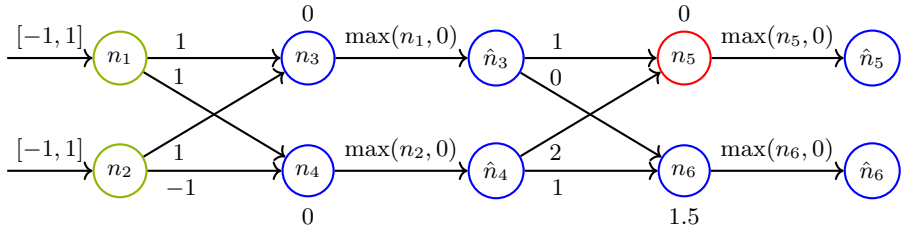
\begin{figure}[t!]
	\centering
	\begin{tikzpicture}
		
		\node[circle, draw= applegreen, thick, minimum width = 20,
		minimum height = 20] (input1) {$n_1$};

		\node[circle, draw= applegreen, thick, minimum width = 20,
		minimum height = 20] (input2) at ($(input1) + (0,-1.5)$) {$n_2$};
		
		
		\node (hidden10) at ($(input1) + (2.5,0.6)$) {$0$};
		
		\node (hidden20) at ($(input1) + (2.5,-1.5-0.6)$) {$0$};
		
		\node (hidden50) at ($(input1) + (7.5,0.6)$) {$0$};
		
		\node (hidden60) at ($(input1) + (7.5,-1.5-0.6)$) {$1.5$};

		\node[circle, draw= blue, thick, minimum width = 20,
		minimum height = 20] (hidden1) at ($(input1) + (2.5,0)$) {$n_3$};
		\node[circle, draw= blue, thick] (hidden2) at ($(input1) + (2.5,-1.5)$) {$n_4$};
		
		\node[circle, draw= blue, thick, minimum width = 20,
		minimum height = 20] (hidden3) at ($(input1) + (5,0)$){$\hat{n}_3$};
		\node[circle, draw= blue, thick] (hidden4) at ($(input1) + (5,-1.5)$) {$\hat{n}_4$};

		\node[circle, draw= red, thick, minimum width = 20,
		minimum height = 20] (hidden5) at ($(input1) + (7.5,0)$){$n_5$};
		\node[circle, draw= blue, thick] (hidden6) at ($(input1) + (7.5,-1.5)$) {$n_6$};

		\node[circle, draw= blue, thick, minimum width = 20,
		minimum height = 20] (output1) at ($(input1) + (10,0)$){$\hat{n}_5$};
		
		\node[circle, draw= blue, thick, minimum width = 20,
		minimum height = 20] (output2) at ($(input1) + (10,-1.5)$){$\hat{n}_{6}$};

		\draw[->,thick] ($(input1) + (-1.5,0)$) -- (input1) node[midway, above] {$[-1,1]$};
		
		\draw[->,thick] ($(input1) + (-1.5,-1.5)$) -- (input2) node[midway, above] {$[-1,1]$};

		\draw[->,thick] (input1) -- (hidden1) node[near start, above] {$1$};
		\draw[->,thick] (input1) -- (hidden2)node[near start, above] {$1$};
		
		\draw[->, thick] (input2) -- (hidden1) node[near start, below] {$1$};
		\draw[->, thick] (input2) -- (hidden2)node[near start, below] {$-1$};

		\draw[->, thick] (hidden1) -- (hidden3) node[midway, above] {$\max(n_1,0)$};
		\draw[->, thick] (hidden2) -- (hidden4) node[midway, above] {$\max(n_2,0)$};

		\draw[->,thick] (hidden3) -- (hidden5) node[near start, above] {$1$};			
		\draw[->,thick] (hidden3) -- (hidden6) node[near start, above] {$0$};
		
		\draw[->,thick] (hidden4) -- (hidden5)node[near start, below] {$2$};
		\draw[->,thick] (hidden4) -- (hidden6)node[near start, below] {$1$};

		\draw[->,thick] (hidden5) -- (output1) node[midway, above] {$\max(n_5,0)$};
		\draw[->,thick] (hidden6) -- (output2) node[midway, above] {$\max(n_6,0)$};


	\end{tikzpicture}
	\caption{A DNN. Every neuron is separated into 2 nodes, 
	$n$ pre- and $\hat{n}$ post-ReLU activation.} 
	\label{fig1}
\end{figure}

Consider the simpler abstraction, termed ``Box abstraction", recalled e.g. in \cite{deeppoly}: it inductively computes the bounds for each neuron in the subsequent layer independently. This is achieved by considering the weighted sum of the bounds from the previous layer, followed by clipping the lower bound at $\max(0,$ lower bound$)$ to represent the ReLU function, and so forth. 
For all $i \geq 3$, define $x_i=\val_{\vx}(n_i)$, where $\vx=(x_1,x_2)$.
Taking the DNN example from Fig \ref{fig1}, assume $x_1,x_2 \in [-1,1]$. This implies that $x_3,x_4 \in [-2,2]$. After applying the ReLU function, $\hat{x}_3,\hat{x}_4$ are constrained to {$[1.5,3.5]$} , leading to $x_5 \in [0,6]$ and $x_6 \in [0,2]$. 
The bounds for $n_1, \ldots, n_4$ are exact, meaning for every $\alpha$ within the range, an input $\vy$ can be found such that $\val_{\vy}(n_i)=\alpha$. However, this precision is lost from the next layer (beginning with $n_5, n_6$) due to potential dependencies among preceding neurons. For example, it is impossible for $x_5=\Val_{\vx}(n_5)$ to reach $6$, as it would necessitate both $x_3=2$ and $x_4=2$, which is not possible at the same time as 
$x_3=2$ implies $x_1=x_2=1$ and $x_4=2$ implies $x_2=-1$ (and $x_1=1$), a contradiction.

In \cite{DBLP_Ehlers17,deeppoly} and others, the {\em triangular abstraction} was proposed:
\begin{align}
	\ReLU(x) = \max (0,x) \leq \hat{x} \leq \UB \frac{x-\LB}{\UB-\LB} 
	\label{eq:deeppoly}
\end{align} 
It has two lower bounds (the 0 and identity functions), and one upper bound. 
DeepPoly \cite{deeppoly} chooses one of the two lower bounds for each neuron $x$, giving rise to a greedy quadratic-time algorithm to compute very fast an abstraction of the value of $\hat{x}$ (but not that accurately).

\subsection{MILP and LP encodings for DNNs}

At the other end of the spectrum, we find the Mixed Integer Linear Programming (MILP) value abstraction, which is a complete (but inefficient) method. 
Consider an unstable neuron $n$, whose value $x \in [\LB,\UB]$ with $\LB<0<\UB$.
The value $\hat{x}$ of $\ReLU(x)$ can be encoded exactly in an MILP formula with one 
integer (actually even binary) variable $a$ valued in $\{0,1\}$, using constants $\UB,\LB$ with 4 constraints \cite{MILP}:

\vspace{-0.4cm}
\begin{equation}\quad \hat{x} \geq x \quad \wedge \quad \hat{x} \geq 0 \quad \wedge \quad \hat{x} \leq \UB \cdot a \quad \wedge \quad \hat{x} \leq x-\LB \cdot (1-a)
\label{eq11}
\end{equation}

For all $x \in [\LB,\UB] \setminus 0$, there exists a unique solution $(a,\hat{x})$ that meets these constraints, with $\hat{x}=\ReLU(x)$ \cite{MILP}. The value of $a$ is 0 if $x < 0$, and 1 if $x>0$, and can be either if $x=0$. This encoding approach can be applied to every (unstable) ReLU node, and optimizing its value can help getting more accurate bounds. However, for networks with hundreds of {\em unstable} nodes, the resulting MILP formulation will contain numerous {binary variables} and generally bounds obtained will not be accurate, even using powerful commercial solvers such as Gurobi.

MILP instances can be linearly relaxed into LP over-abstraction, where variables originally restricted to integers in $\{0,1\}$ (binary) are relaxed to real numbers in the interval $[0,1]$, while maintaining the same encoding. As solving LP instances is polynomial time, this optimization is significantly more efficient. However, this efficiency comes at the cost of precision, often resulting in less stringent bounds. This approach is termed the {\em LP abstraction}. We invoke a folklore result on the LP relaxation of (\ref{eq11}), for which we provide a direct and explicit proof:


\begin{proposition}
	\cite{alessandro}
	\label{LP}
	The LP relaxation of (\ref{eq11}) is equivalent with the triangular abstraction 	(\ref{eq:deeppoly}).
\end{proposition}
 

\begin{proof}
	Consider an unstable neuron $n$, that is $\LB < 0 < \UB$.
The lower bound on $\hat{x}$ is simple, as $\hat{x} \geq 0 \wedge \hat{x} \geq x$ is immediately equivalent with $\hat{x} \geq \ReLU(x)$.

We now show that the three constraints 
$\hat{x} \leq \UB \cdot a \, \wedge \, \hat{x} \leq x-\LB \cdot (1-a) \, \wedge \, a \in [0,1]$ translates into $\hat{x} \leq \UB \frac{x-\LB}{\UB-\LB}$. 
We have $\hat{x}$ is upper bounded by $max_{a \in [0,1]} (min(\UB \cdot a, x - \LB (1-a)))$, and this bound can be reached. Furthermore, using standard function analysis tools (derivative...), we can show that the function $a \mapsto \min(\UB \cdot a, x - \LB (1-a))$ attains its maximum when $\UB \cdot a = x - \LB (1-a)$, leading to the equation $(\UB - \LB) a = x - \LB$ and consequently $a = \frac{x - \LB}{\UB-\LB}$. This results in an upper bound $\hat{x} \leq \UB \frac{x - \LB}{\UB-\LB}$, which can be reached, hence the equivalence.
\qed
\end{proof}

\section{Partial MILP}

In this paper, we revisit the use of {\em partial MILP} (pMILP for short, see \cite{DivideAndSlide}), to get interesting trade-offs between accuracy and runtime.
Let $X$ be a subset of the set of unstable neurons, and $n$ a neuron for which we want to compute upper and lower bounds on values: pMILP$_X(n)$ simply calls Gurobi to minimize/maximize the value of $n$ with the MILP model encoding (equation (2) Section 3.2), where variable $a$ is:
\begin{itemize}
\item binary for neurons in $X$ (exact encoding of the ReLU),
\item linear for neurons not in $X$ (linear relaxation).
\end{itemize}

Formally, we denote by MILP$_X$ the MILP encoding where the variable $a$ encoding $\ReLU(y)$ is binary for $y\in X$, and linear for $y\notin X$ ($y$ being unstable or stable). 
We say that nodes in $X$ are {\em opened}. 
That is, if $X$ is a strict subset of the set of all unstable neurons, then 
MILP$_X$ is an abstraction of the constraints in the full MILP model. If $X$ covers all unstable neurons, then MILP$_X$ is exact as there is only one linear mode for stable variables. The worst-case complexity to solve MILP$_X$ is NP-complete in $|X|$, ie the number of binary variables.

\smallskip

To reduce the runtime, we will limit the size of subset $X$. This a priori hurts accuracy. To recover some of this accuracy, we use an iterative approach similar to the box abstraction or DeepPoly \cite{deeppoly}, computing lower and upper bounds $\LB,\UB$ for neurons $n$ of {each layer} iteratively, that are used when computing values of the next layer.
So we trade-off the NP-complexity in $|X|$ with a linear complexity in the number of neurons. Notice that all the neurons of a layer can be treated in parallel.

\medskip


\SetKwInput{KwInput}{Input}
\SetKwInput{KwOutput}{Output}

\begin{algorithm}[t!]
	\caption{pMILP$_K$}
	\label{algo1}
	\KwInput{Bounds $[\LB(m),\UB(m)]$ for nodes $m$ at layer $0$}
	
	\KwOutput{Bounds $[\LB,\UB]$ for every node $n$}
	
	\For{layer $k=1, \ldots, \ell$}{
		\For{neuron $n$ in layer $k$}{
			
			Compute $X$ a set of $K$ nodes most important for target neuron $n$.
			
			Run pMILP$_X(n)$ to obtain $[\LB(n),\UB(n)]$ using $\MILP_X$ and additional constraints $\bigvee_{m \text{ in layer } <k} \val(m) \in [\LB(m),\UB(m)]$.
		}
	}
\end{algorithm}	


We provide the pseudo code for pMILP$_K$ in Algorithm \ref{algo1}.
pMILP$_K$ has a worst case complexity bounded by $O(N \cdot \MILP(N,K))$, 
where $N$ is the number of nodes of the DNN, and $\MILP(N,K)$ is the complexity of solving a MILP program with $K$ {binary variables} and $N$ linear variables.
We have $\MILP(N,K) \leq 2^K \LP(N)$ where $\LP(N)$ is the Polynomial time to solve a Linear Program with $N$ variables, $2^K$ being an upper bound. Solvers like Gurobi are quite adept and usually do not need to evaluate all $2^K$ ReLU configurations to deduce the bounds.
It is worth mentioning that the "for" loop iterating over neurons $n$ in layer $k$ (line 2) can be executed in parallel, because the computation only depends on bounds from preceding layers, not the current layer $k$. This worst-case complexity compares favorably when $K<<N$ with the worst-case complexity of BaB which would be $O(out \times 2^N)$, where $out$ is the number of output neurons. Of course, actual complexity is always much better than the worst case, thanks to different heuristics, such as pruning, which is what happens for easy instances, for which the linear cost $N$ in pMILP (which cannot be avoided) is actually the limiting factor. Only experimental results can tell whether the accuracy lost using reduced variables in pMILP can be a good trade off vs efficiency, both depending on the number $K$ chosen.

If $K$ is sufficiently small, 
this approach is expected to be efficient. 
The crucial factor in such an approach is to {\em select} few opened ReLU nodes in $X$ which are the most important for the accuracy. An extreme strategy was adopted in 
\cite{DivideAndSlide}, where only ReLU nodes of the immediate previous layer can be opened, and the measure to choose ReLU $a$ when computing the bounds for neuron $b$ was to consider $|W_{ab}| (\UB(a)-\LB(a))$.

%% file: Comparison.tex
\section{Global Scoring Functions}

\label{sec4p5}


To select nodes $X$ for pMILP$_X$, we first adapt from the BaB-SR \cite{BaB} and FSB \cite{FSB} functions, 
 which originally iteratively select one node to branch on for BaB.
 Intuitively, we extract the scoring $s_{SR}$ and $s_{FSB}$ from both BaB-SR and FSB, as the 
 BaB bounding step is not adapted for pMILP. 
 We will call such functions {\em  global scoring} (GS) functions.

 They both work by backpropagating gradients vectors $\bm{\lambda}$ from the neurons under consideration in layer $n$, back to neurons to be potentially selected. To do so, they consider the rate of $\ReLU(\bm{u}_k)$ to be 
$r(\bm{u}_k)=\frac{\max(0,\UB(\bm{u}_k))}{\max(0,\UB(\bm{u}_k))-\min(0,\LB(\bm{u}_k))} \in [0,1]$, 
with $r(b)=0$ iff $\UB(b)\leq 0$ and $r(b)=1$ iff $\LB(b)\geq 0$.

\begin{align}
\bm{\lambda}_{n-1} = -{(\bm{W}^n)}^T\bm{1}, \hspace*{4ex}  	\bm{\lambda}_{k-1} = {(\bm{W}^k)}^T\big( r(\bm{u}_k) \odot\bm{\lambda}_{k}\big) \hspace*{4ex}  k\in [n-1,2]
\end{align}

Then, the scoring functions $\bm{s}_{SR}$ and $\bm{s}_{FSB}$ for ReLUs in layer $k$ are computed by approximating how each node would impact the neurons in layer $n$, using the computed $\bm{\lambda}$, differeing only in how they handle the bias $\bm{b}_k$:

\begin{align*}
	\bm{s}_{SR}(k) =& 
	\Biggl\lvert r(\bm{u}_k) \odot \LB(\bm{u}_k) \odot \max(\bm{\lambda}_{k},0)
	+ \max\{0,\bm{\lambda}_{k}\odot\bm{b}_{k}\}-r(\bm{u}_k) \odot\bm{\lambda}_{k}\odot\bm{{b}}_{k}
	\Biggr\rvert  \\
	\bm{s}_{FSB}(k) =& \Biggl\lvert r(\bm{u}_k) \odot \LB(\bm{u}_k) \odot \max(\bm{\lambda}_{k},0)
	+ \min\{0,\bm{\lambda}_{k}\odot\bm{b}_{k}\}-r(\bm{u}_k) \odot\bm{\lambda}_{k}\odot\bm{{b}}_{k}
	\Biggr\rvert
\end{align*}

	Then, ReLU are ranked using these scores, to select the most important unstable ReLUs
	(with $\LB(u_k) < 0 < \UB(u_k)$).

\noindent {\em Running Example:} Consider Fig. \ref{img:FSB_example}. 
It has no bias, so $s_{SR}=s_{FSB}$. We have $\lambda(b)=d$, $\lambda(a)=\frac{cd}{2}$.
The value of $s_{FSB}(a)$ depends on the signs of $c,d$.

We perform a novel comparison between $s_{FSB}(a)$ and $\Delta(z)$, the difference on the {\em maximal bound} computed by pMILP$_X(z)$ when opening the $\ReLU$ of node $a$ ($X=\{a\}$), yielding $\val(\hat{a}) = \ReLU(\val(a))$, versus having $X=\emptyset$, for which 
$\val(\hat{a})$ can be any value in the triangle approximation (Prop. \ref{LP}).

The most favorable cases are $c < 0 < d$ and $d < 0 < c$: 
as $cd <0$, we have $s_{FSB}(a)=\max(0,\frac{cd}{4}) = 0$.
Because $cd <0$, both $a$ and $\hat{a}$ need to be minimized by MILP$_X$ in order to maximize the value of $z$. For $X=\emptyset$, the LP approximation (=MILP$_\emptyset$) will thus 
set $\val(\hat{a}) = \ReLU(\val(a))$ as this is the minimum value in the triangle approximation. Notice that opening $X=\{a\}$ yields the same $\val(\hat{a}) = \ReLU(\val(a))$.
That is, opening node $a$ is not improving the bound $\UB(z)$, as correctly predicted by the score $s_{FSB}(a)= \Delta(a) = 0$.

Case $c>0,d>0$: we have $s_{FSB}(a)=\frac{cd}{4}$.
The value of $a,\hat{a},b,\hat{b}$ should be maximized to maximize the value of $z$, because $W_{bz}=d>0$ and $W_{ab}=c>0$. 
Now, let us call $val(a)$ the maximum value for $a$ 
(same for MILP$_\emptyset$ and MILP$_{\{a\}}$).
The maximum value for $\hat{a}$ under LP ($X=\emptyset$) is 
$\frac{1}{2} \cdot (\val(a)-LB(a))$ following the triangle approximation. 
Now, if $X=\{a\}$, then the value of $\hat{a}$ is $\ReLU(\val(a))$: the difference 
$\Delta(\hat{a})$ is between 0 (for $\val(a)\in \{LB,UB\}$) and $\frac{1}{2}$ (for $\val(a)=0$). The difference $\Delta(b)$ for the value of $b$ is thus between 0 and $\frac{c}{2}$, which means $\Delta(\hat{b}) \in [0, \frac{c}{4}]$, using the upper function of the triangle approximation of rate $r(b)=\frac{1}{2}$, as the value of $\hat{b}$ should be maximized.
This means a difference $\Delta(z) \in [0, \frac{cd}{4}]$ 
(depending on the maximal value $val(a)$), to compare with the fixed 
$s_{FSB}(a)=\frac{cd}{4}$.

\begin{figure}[t!]
	\centering
	\begin{tikzpicture}
		
		\node (input1) {};
		
		
		
		\node (hidden10) at ($(input1) + (2,0.6)$) {$[-1,1]$};
		
		\node (hidden20) at ($(input1) + (2.5,-1.5-0.6)$) {};
		
		\node (hidden50) at ($(input1) + (7.5,0.6)$) {$[-1,1]$};
		
		\node (hidden60) at ($(input1) + (7.5,-1.5-0.6)$) {};

		\node[circle, draw= black, thick, minimum width = 20,
		minimum height = 20] (hidden1) at ($(input1) + (2,0)$) {$a$};
		
		\node[circle, draw= black, thick, minimum width = 20,
		minimum height = 20] (hidden3) at ($(input1) + (5,0)$){$\hat{a}$};
		\node (hidden4) at ($(input1) + (5.5,-1)$) {};

		\node[circle, draw= black, thick, minimum width = 20,
		minimum height = 20] (hidden5) at ($(input1) + (7.5,0)$){$b$};

		\node[circle, draw= black, thick, minimum width = 20,
		minimum height = 20] (output1) at ($(input1) + (10,0)$){$\hat{b}$};
		
		\node (output2) at ($(input1) + (10.5,-1)$){};
		
			\node[circle, draw= black, thick, minimum width = 20,
		minimum height = 20] (output3) at ($(input1) + (12.5,0)$){$z$};

		
		
		
		\draw[->,thick] (input1) -- (hidden1) node[midway, above] {};
		

		\draw[->, thick] (hidden1) -- (hidden3) node[midway, above] {open ReLU?};

		\draw[->, thick] (hidden3) -- (hidden5) node[midway, above] {{\color{red}$c$}};			
	
		\draw[->,thick] (hidden4) -- (hidden5)node[near start, above] {};

		\draw[->,thick] (hidden5) -- (output1) node[midway, above] {LP ReLU};
		
		\draw[->,thick] (output1) -- (output3)node[midway, above] {{\color{red}$d$}};
		\draw[->,thick] (output2) -- (output3)node[near start, below] {};

	\end{tikzpicture}
	\vspace{-0.8cm}
	\caption{A running example with parametric weights {\color{red}$c$} and {\color{red}$d$}}
	\label{img:FSB_example}
\end{figure}
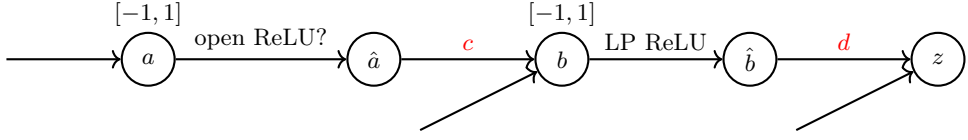

The last case is the most problematic: 
Case $c<0,d<0$, implying that $s_{FSB}(a)=\frac{cd}{4}$, because $cd >0$.
Opening $\ReLU(a)$ will have the same impact on the value of $\hat{a},b$ 
than in case $c>0,d>0$, with $\Delta(b) \in [0,\frac{c}{2}]$.
However, as $d<0$, value of $\hat{b}$ needs to be minimized to maximize the value of $z$. That is, the value of $\hat{b}$ will be the value of $\ReLU(sol(b))$, and the change 
$\Delta(\hat{b})$ will be either 0 in case $\val(b)<0$,
or $\Delta(b) \in [0,\frac{c}{2}]$ for $\val(b)>0$.
That is, $z$ will be modified by either $\Delta(z)=0$ or $\Delta(z) = d \Delta(b)$, to be compared with the fix value $s_{FSB}(a)=\frac{cd}{4}$, which is not always an overapproximation: we have $\Delta(z)=\frac{cd}{2} > s_{FSB}(a)$, 
if $\val(a)=0$ and $\val(b)>0$.

We call {\em global scoring} (GS) these functions $s_{FSB},s_{SR}$  because they score ReLUs as {accurately} as possible, considering that they do not have access to the values 
$\val(a),\val(b)$ maximizing $z$. Following this analysis of $s_{FSB},s_{SR}$, 
next Section presents a novel scoring function more accurate wrt $\Delta(z)$.

%% file: formula.tex
\section{Solution-Aware Scoring.}

\label{sec4}

In this section, we propose {\em Solution-Aware Scoring} (SAS),
to evaluate accurately how opening a ReLU impacts the accuracy.
To do so, SAS considers explicitly a solution to a unique LP call, which is reasonably fast to obtain as there is no binary variables (polynomial time). 
Assume that we want to compute an upper bound for neuron $z$ on layer $\ell_z$.
We write $n < z$ if neuron $n$ is on a layer before $\ell_z$, and $n \leq z$ if $n< z$ or $n=z$. We denote ($\Sol\_\max_X^z(n))_{n \leq z}$ a solution of $\mathcal{M}_X$ maximizing $z$: $\Sol\_\max_X^z(z)$ is the maximum of $z$ under $\mathcal{M}_X$.

Consider $(\sol(n))_{n \leq z} = (\Sol\_\max_\emptyset^z(n))_{n \leq z}$, a solution maximizing the value for $z$ when all ReLU use the LP relaxation.
Function
$\Improve\_\max^z(n)=$ $\sol(z) - \Sol\_\max_{\{n\}}^z(z)$, 
accurately represents how much opening neuron $n < z$ reduces the maximum computed for $z$
compared with using only LP. 
We have $\Improve\_\max^z(n)\geq 0$ as $\Sol\_\max_{\{n\}}^z$ fulfills all the constraints of 
$\mathcal{M}_\emptyset$, so $\Sol\_\max_{\{n\}}^z(z) \leq \sol(z)$.
Computing exactly $\Improve\_\max^z(n)$ would need a MILP call on $\mathcal{M}_{\{n\}}$ for every neuron $n \leq z$, which would be very time consuming. Instead, the SAS function uses a (single) LP call to compute $(\sol(n))_{n \leq z}$, with negligible runtime wrt the forthcoming  $\MILP_X$ call, and yet accurately approximates $\Improve\_\max^z(n)$ 
(Fig. \ref{fig_table3}).

For a neuron $b$ on the layer before layer $\ell_z$, we define:

\vspace{-0.5cm}
\begin{align}
		\Utility\_\max\nolimits^z(b) = W_{bz} \times (\sol(\hat{b})- \ReLU(\sol(b)))
\end{align}
\vspace{-0.5cm}
	

\noindent {\em Comparison:} consider $b$ with $W_{bz}<0$: 
to maximize $z$, the value of $\sol(\hat{b})$ is minimized by LP, 
ie $\sol(\hat{b})=\ReLU(\sol(b))$ thanks to Proposition~\ref{LP}. 
Thus, we have $\Utility\_\max^z(b)=0=\Improve_{max}^z(b)$.
Notice that the original scoring function $|W_{bz}|(\UB(b)-\LB(b))$ \cite{DivideAndSlide}  would be possibly very large in this case. However, {\sf GS} scoring functions from BaB-SR and FSB would also accurately compute  $s_{FSB}(b)=s_{SR}=0$.
Notice that $\Utility$ does not need to consider bias explicitly, unlike $s_{FSB},s_{SR}$,
as they are already accounted for in the solution considered. 

\medskip

Consider a neuron $a$ two layers before $\ell_z$, 
$b$ denoting neurons in the layer $\ell$ just before $\ell_z$.
Recall the rate $r(b)=\frac{\max(0,\UB(b))}{\max(0,\UB(b))-\min(0,\LB(b))} \in [0,1]$.
We define:

\newpage

\begin{flalign}
	\Delta(\hat{a}) &= \ReLU(\sol(a))-\sol(\hat{a})&&\\
	\forall b \in \ell, \Delta(b) &= W_{ab}\Delta(\hat{a})&&
\end{flalign}

\vspace{-0.6cm}

\begin{subnumcases}{\forall b \in \ell, \Delta(\hat{b}) =}
		r(b)\Delta(b), & for $W_{bz} > 0$ \\
		\max(\Delta(b),-\sol(b)), & for $W_{bz} < 0$ and $\sol(b)\geq0$\\
		\max(0,\Delta(b)+\sol(b)), & for $W_{bz} < 0$ and $\sol(b)<0$ \quad \, \quad \, \quad		 
\end{subnumcases}
\vspace{-0.4cm}
\begin{flalign}
	\Utility\_\max\nolimits^z(a) &= \Delta(z) = -\sum_{b \in \ell} W_{bz} \Delta(\hat{b})&&
\end{flalign}

\begin{figure}[t!]
	\begin{centering}
	\begin{tikzpicture}[scale=1, >=stealth]
		
		\draw[->] (-5,0) -- (4,0) node[right] {$b$};
		\draw[->] (0,-1) -- (0,3) node[above] {$\hat{b}$};
		
		\draw[line width=0.4mm, blue] (-3,0) -- (0,0);
		\draw[thick, blue] (0,0) -- (2.5,2.5) node[below, shift={(0.35,-0.55)}] {$\hat{b} = \ReLU(b)$};
		\draw[thick, blue] (-3,0) -- (2.5,2.5) node[above, shift={(-5.7,-2)}] {$\hat{b} = r(b) (b-\LB)$};
		
		\draw[dashed] (2.5,0) -- (2.5,2.5) -- (0,2.5); 
		\node[below left] at (0,0) {};
		
		
		\foreach \x in {2.5}
		\draw[shift={(\x,0)}] (0,0.1) -- (0,-0.1) node[below] {$\UB$};
		\foreach \x in {-3}
		\draw[shift={(\x,0)}] (0,0.1) -- (0,-0.1) node[below] {$\LB$};
		
		\foreach \y in {2.5}
		\draw[shift={(0,\y)}] (0.1,0) -- (-0.1,0) node[left] {$\UB$};
		
		\draw[<->, thick] (0.1, 0.1) -- (0.9, 0.9) node[above,shift={(0.65,-0.45)}] {$case$ (b)}
		node[below,shift={(-0.1,-1.10)}] {$sol(b)\Delta(b)$}
		node[above,shift={(-1.35,-0.55)}] {$\Delta(\hat{b})$};
		\filldraw[black] (0.5, 0.5) circle (1.5pt);
		\filldraw[black] (0.5, 0) circle (1.5pt);
		\draw[<->, thick] (0.1, 0) -- (0.9, 0);
		\draw[dashed] (0.5,0) -- (0.5,1.6); 

		\draw[<->, thick] (-0.5, 0) -- (-1.5, 0) node[above,shift={(0.6,0)}] {$case$ (c)}
		node[below,shift={(0.6,-0.15)}] {$sol(b)\Delta(b)$};
		\filldraw[black] (-1, 0) circle (1.5pt);

		\draw[<->, thick] (0.1, 1.375) -- (0.9, 1.75) node[above,shift={(-0.3,0)}] {$case$ (a)}
		node[above,shift={(-1.4,-0.45)}] {$\Delta(\hat{b})$};
		\filldraw[black] (0.5, 1.6) circle (1.5pt);
		
		\draw[<->, thick] (0, 0.1) -- (0, 0.9);
		
		\draw[<->, thick] (0, 1.375) -- (0, 1.75);

	\end{tikzpicture}
	\caption{Different cases for $\ReLU(b)$
}
	\label{node:b}
\end{centering}
\end{figure}

%
%
%
%



\noindent {\em Comparison:} First, the original \cite{DivideAndSlide} does not propose a formula for node $a$ two layers before $z$. So we will compare {\sf SAS} with {\sf GS}.
Consider again the running example of Fig. \ref{img:FSB_example}. 

In the case $c\cdot d<0$, we have $\Utility\_\max\nolimits^z(a)=\Improve_{max}^z(b)=s_{FSB}(a)=s_{SR}(a)=0$, as $\Delta(\hat{a})=0$. 

In the case $c>0,d>0$, 
$\Delta(\hat{a}) = \ReLU(\sol(a))-\sol(\hat{a})$ is precise,
whereas the corresponding $\Delta_{FSB}(\hat{a}) = \LB(a) r(a)$ is only an upperbound.

The last case $c<0,d<0$ is the most extreme: 
$\Delta(\hat{b})$ adapts to the case (b),(c) in Fig. \ref{node:b} leveraging the value $\sol(b)$, which yields very different values, whereas the corresponding $\Delta_{FSB}(\hat{b})$
is always $r(b) \Delta_{FSB}(b)$: 
\begin{itemize}
\item For $sol(b)<<0$, we will have 
$\Delta(\hat{b})=0<\Delta_{FSB}(\hat{b})$.
\item For $sol(b)>>0$, we will have 
$\Delta(\hat{b})=\Delta(b) > \frac{1}{2}\Delta_{FSB}(\hat{b})$ 
as $r(b)=\frac{1}{2}$.
\end{itemize}

Further, we can show that $\Utility$ is a safe overapproximation of 
$\Improve\_\max^z(a)$, which does not hold for $s_{FSB},s_{SR}$ (because of the case $sol(b)>>0$):

\begin{proposition}
	\label{prop2}
		$0 \leq \Improve\_\max^z(a) \leq \Utility\_\max^z(a)$. 
\end{proposition}

In particular, for all nodes $a$ with $\Utility\_\max\nolimits^z(a)=0$, 
we are sure that this node is not having any impact on $\Sol\_\max_{\{a\}}^z(z)$. 

%
	
%

	\begin{proof}
    Consider $\sol'(n)_{n \leq z}$ with
	$\sol'(n)=\sol(n)$ for all $n \notin \{z,\hat{a}\} \cup \{b,\hat{b} \mid b \in \ell\}$. In particular,  $\sol'(a) = \sol(a)$.
	Now, define $\sol'(\hat{a}) = \ReLU(\sol(a))$. 
	That is, $\sol'(\hat{a})$ is the correct value for $\hat{a}$, obtained if we open neuron $a$, compared to the LP abstraction for $\sol(\hat{a})$.
	We define $\sol'(b)=\sol(b)+\Delta(b)$ and 
	$\sol'(\hat{b})=\sol(\hat{b}) + \Delta(\hat{b})$.
	Last, $\sol'(z)=\sol(z) + \sum_{b \in \ell} W_{bz} \Delta(\hat{b})$.
	We will show:
	\begin{equation}
		\label{eq12}
		(\sol'(n))_{n \leq z} \text{ satisfies the constraints in } \mathcal{M}_{\{a\}}
	\end{equation} 
	This suffices to conclude: as
	$\sol'(z)$ is a solution of $\mathcal{M}_{\{a\}}$, it is smaller or equal to the maximal solution: $\sol'(z) \leq$ $\Sol\_\max_{\{a\}}^z(z)$. That is, 
	$\sol(z)-\sol'(z) \geq \sol(z) -$ $\Sol\_\max_{\{a\}}^z(z)$, i.e. 
	$ \Utility\_\max^z(a) \geq \Improve\_\max^z(a)$.
	In particular, we have that $\Utility\_\max^z(a) \geq 0$, which was not obvious from the definition.	

Finally, we show (\ref{eq12}). First, opening $a$ changes the value of $\hat{a}$ from
	$\sol(\hat{a})$ to $\ReLU(\sol(a)) = sol(\hat{a}) + \Delta(a)$, 
	and from $sol(b)$ to $sol(b) + \Delta(b)$.
	The case of $\Delta(\hat{b})$ is the most interesting:
	If (a) $W_{bz}>0$, to maximize $z$, the LP solver sets $\sol(\hat{b})$ to the maximal possible value, which is 
	$r(b) \sol(b)+$ Cst according to Proposition \ref{LP}.
Changing $b$ by $\Delta(b)$ thus results in changing $\sol(\hat{b})$ by 
$r(b) \Delta(b)$.
If $W_{bz}\leq0$, then the LP solver sets $\sol(\hat{b})$ to the lowest possible value to maximize $z$, which happens to be $\ReLU(b)$ according to Proposition \ref{LP}.
If (b) $\sol(b) > 0$, then 
$\sol(\hat{b})=\ReLU(\sol(b))=\sol(b)$, and the change to $\hat{b}$ will be 
the full $\Delta(b)$, unless $\Delta(b) < -\sol(b) < 0$ in which case it is 
$-\sol(b)$.
If (c) $\sol(b) < 0$, then we have $\sol(\hat{b})=\ReLU(b)=0$ and opening $a$ moves away 
from 0 only if $\sol(b)+\Delta(b)>0$. 
\qed
\end{proof}

%% file: experiments.tex
\section{Experimental Evaluation}

We implemented \toolname\ in Python 3.8: it first calls $\alpha,\beta$-CROWN with small time out (10s for small and 30s for larger DNNs), and then call pMILP on the undecided inputs.
Gurobi 9.52 was used for solving LP and MILP problems. We conducted our evaluation on an AMD Threadripper 7970X  ($32$ cores$@4.0$GHz) with 256 GB of main memory and 2 NVIDIA RTX 4090. The  objectives of our evaluation was to answer the following  questions:
\vspace{-0.1cm}

\begin{enumerate}
	\item How does the the choice of the set $X$ impacts the accuracy of $\MILP_X$? 
	\item How accurate is Hybrid MILP, and how efficient is it?
\end{enumerate}

\subsection{Comparison between different scoring functions (accuracy).}

\begin{figure}[t!]
	\centering
	\includegraphics[height=7cm]{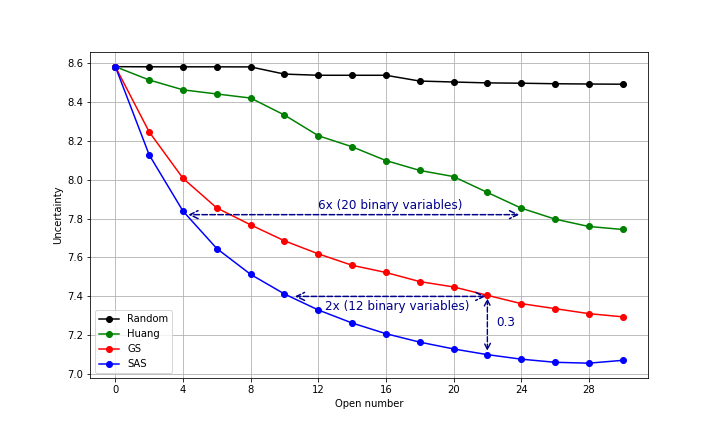}.
	\caption{Average uncertainty of pMILP for neurons of the fourth layer of CNN-B-Adv, 
	when selecting between 0 and 30 ReLUs with different scoring functions.}
	\label{fig_table3}
\end{figure}


To measure the impact of the scoring function to select neurons to open, 
we considered the complex CNN-B-Adv, which has 16634 nodes. 
We tested over the $\vx=85$th image in the CIFAR-10 dataset.
To measure the accuracy, we measure the uncertainty of all nodes in the 4th layer:
the uncertainty of a node is the range between its computed lower and upper bound. 
We then average the uncertainty among all the nodes of the layer.
Formally, the uncertainty of a node $a$ with bounds $[\LB,\UB]$ is uncert$(a) = \UB(a) - \LB(a)$. The average uncertainty of layer $\ell$ is 
$\dfrac{\sum_{a\in l} \text{uncert}(a)}{size(\ell)}.$

We report in Figure \ref{fig_table3} the average uncertainty of $\MILP_X$ following the 
choice of the $K$ heaviest neurons of {\sf SAS}, compared with a random choice, with Huang \cite{DivideAndSlide}, based on strength$(n) = (\UB(n)-\LB(n))\cdot |W_{nz}|$, and with {\sf GS} (here, $s_{FSB}$).

\begin{table}[b!]	
	\centering
	\begin{tabular}{|c||c c|c c|c c|c c|}
		\hline
		  Image 85& \multicolumn{2}{c|}{45 open ReLUs} & \multicolumn{2}{c|}{50 open ReLUs} 
		  & \multicolumn{2}{c|}{55 open ReLUs} & \multicolumn{2}{c|}{60 open ReLUs}\\ 
		 order & distance & time & distance & time & distance & time & distance & time \\
		\hline \hline
		Gurobi & 0.384 & 252s & 0.368&265s  &0.342 & 287s& 0.315  & 306s  \\ 
		{\sf SAS} (static) & 0.384 & 448s & 0.368&759s  &0.365 &971s & 0.355  & 971s \\
		{\sf GS} (static) & 0.384 & {\bf 251s} &0.368 &{\bf 253s}  & 0.342 &{\bf 270s} & 0.315 & {\bf 281s} \\ \hline
	\end{tabular}
	\caption{Comparison of different orderings with selection of ReLUs by {\sf SAS}.
	}
	\label{table.order}
\end{table}

Overall, {\sf SAS} is more accurate than GS, more accurate than Huang, more accurate than random choice of variables. {\sf SAS} significantly outperforms other solutions, 
with 2 times less binary variables for the same accuracy (10 vs 22) vs {\sf GS} and 6 times vs Huang \cite{DivideAndSlide} (4 vs 24).
Each node of the 4th layer displays a similar pattern, and the pattern is similar for different images.

\smallskip

Gurobi relies on a Branch and Bound procedure to compute bounds. We report in Table \ref{table.order} experiments on changing the ordering on variables in Gurobi, when the selection of ReLUs is fixed by {\sf SAS} on CNN-B-Adv Image 85. We compared 
the standard Gurobi ordering of variables, which is adaptative in each branch, 
with the static order provided by {\sf SAS}, as well as the static order provided by GS, wrt runtime at a fixed accuracy (we report the {\em distance} to verify the image).

{\sf SAS} ordering is particularly counter-productive because it is accurate only for one branch (the most complicated one), whereas Gurobi can adapt the order to each branch.
However, the GS {\em ordering} (with the {\sf SAS} {\em selection}) is better as it is general and not local to a solution, with better runtime than Gurobi, despite its staticness whereas Gurobi order is adaptative.

\newcolumntype{C}{>{\centering\arraybackslash}X}

\subsection{Comparison with MILP (time and accuracy)}

Restricting the set $X$ of open ReLU nodes potentially hurts accuracy. 
Another strategy could be to still compute inductively accurate bounds for each nodes,
replacing pMILP ({\sf SAS, GS}, etc.) with a full MILP model, and stopping it early after some fixed time-out (100s,250s, $\ldots$, 2000s).

We evaluate in Figure \ref{fig555} the full MILP and different pMILP models on the output layer of CNN-B-Adv. The bounds for hidden layers have been computed with the same {\sf SAS} method.
We compare the {\em distance to verify} the same Image 85, specifically the lower bound of an output neuron, the furthest away from verification.
The curve is not as smooth as in Fig.~\ref{fig_table3}, as a unique neuron is considered rather than an average over neurons.

\begin{figure}[b!]
	\includegraphics[scale=0.5]{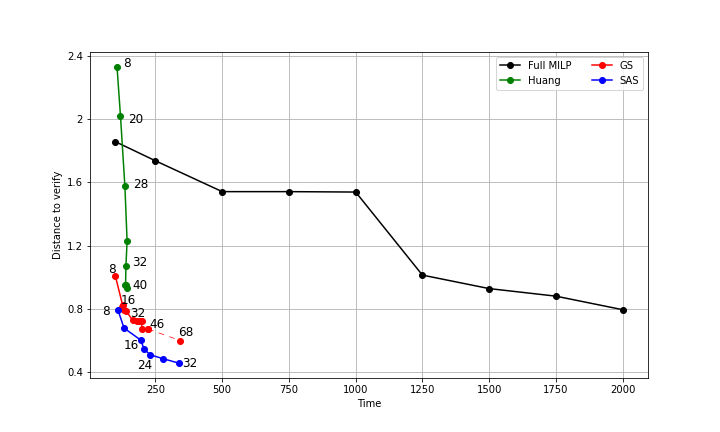}
	\caption{Distance to verify vs runtime: comparison between {\sf SAS}, GS, full MILP and Huang's method for different number of opened ReLUs / time-outs.}
	\label{fig555}
\end{figure}

{\sf SAS} (even {\sf GS}$=s_{FSB}$) is much more accurate than full MILP for every reasonable time-outs considered, with $>10$ times fastest runtime at better accuracy. The reason is that even the advanced MILP solver Gurobi struggles to sort out all ReLUs. It is thus particulary important to have accurate scoring functions as {\sf SAS}, in order to optimize both accuracy and runtime. Huang \cite{DivideAndSlide} is limited in ReLUs in the previous layer, reason why it is stuck at relatively poor accuracy. The accuracy/runtime curve from {\sf GS} is closer to {\sf SAS} than the number of nodes / accuracy curve of Fig. \ref{fig_table3}. This is because many ReLU nodes deemed important by {\sf GS} are not relevant for accuracy, but they are also not penalizing runtime. Still, {\sf SAS} is faster than {\sf GS} at  every accuracy. Importantly, {\sf SAS} is more deterministic in its runtime given a number of ReLU nodes, with half the variance in runtime over different output neurons compared with {\sf GS} for similar accuracy, which helps setting a number of ReLU nodes to open.


\subsection{Ablation study: on usefulness of computing previous layers accurately}	

\begin{figure}[b!]
	\vspace*{-0.8cm}
	\includegraphics[scale=0.5]{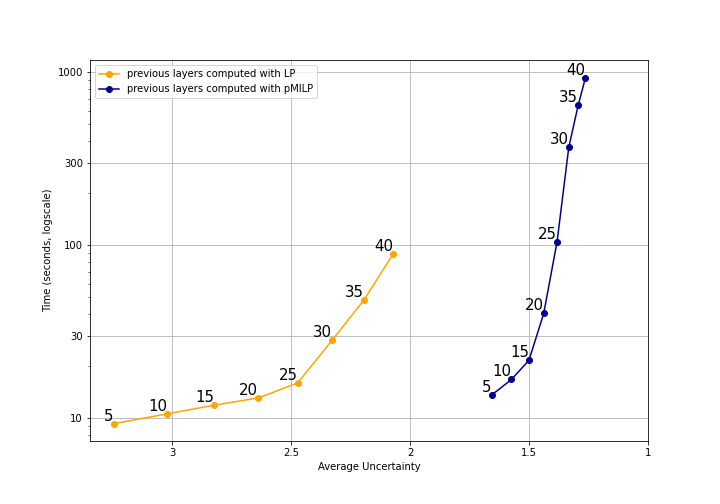}.
    \vspace*{-0.2cm}
	\caption{Comparison of accuracy in layer 3 when bounds for neurons in layer 2 is computed inaccurately using LP vs when bounds of layer 2 are computed accurately using pMILP.
	Time is using logscale.}
	\label{fig3LP}
\end{figure}

We explore the usefulness of computing accurately bounds for each neuron inductively on the layers, 
even on small networks. For that, we consider MNIST $5 \times 100$, computing bounds for nodes of layer 3, comparing when bounds for neurons of layer 2
have been computed inaccurately using LP rather than with the more accurate (partial) MILP.

This experiment explains the rationale to use divide and conquer methodology, 
using many calls (one for each neuron) of pMILP with relatively small number $|X|$ of open ReLUs 
rather than few calls (one per output neuron) of pMILP with larger number $|X|$ of open nodes. 
The benefit is clear already from layer 3, obtaining much tighter bounds (lower uncertainty) when bounds for neurons in layer 2 have been computed accurately using pMILP.

\subsection{Comparison with $\alpha,\beta$-CROWN}

To assess the verification efficiency and runtime vs $\alpha,\beta$-CROWN, we 
conducted evaluations  on neural networks tested in \cite{crown} which are mostly {\em complex} to verify (as easy instances are already appropriately taken care of). Namely, 6 ReLU-DNNs: 5 MNIST DNN that can be found in the \href{https://github.com/eth-sri/eran}{ERAN GitHub} 
(the 4th to the 8th DNNs provided) as well as 1 CIFAR CNN from 
\cite{AdversarialTrainingAndProvableDefenses}, see also \cite{SDPFI}, 
which can be downloaded from the \href{https://github.com/Verified-Intelligence/alpha-beta-CROWN}{$\alpha,\beta$-CROWN GitHub}. 
We commit to the same $\epsilon$ settings as in \cite{crown}, that are recalled in Table \ref{table_beta}. 
For reference, we also report an easy but very large ResNet Network for CIFAR10, already tested with $\alpha,\beta$ CROWN. 
We report in Table \ref{table_hybrid} the $\%$ of undecided images, that is the $\%$ of images than can be neither falsified (by $\alpha,\beta$-CROWN) nor verified by the tested verifier, among the 100 first images for each MNIST or CIFAR10 benchmark. 
The exact same DNNs, image set and $\epsilon$ are used in Tables \ref{table_beta} and \ref{table_hybrid}.

\begin{table}[b!]
	\centering
	\begin{tabular}{||l||c|c|c||c||c||}
		\hline \hline
		 & $\alpha,\beta$-CROWN & $\alpha,\beta$-CROWN & $\alpha,\beta$-CROWN & Refined & \bf Hybrid \\ 
		 Network & TO=10s & TO=30s & TO=2000s & $\beta$-CROWN & \bf MILP \\ 
		\hline
		MNIST 5$\times$100 & 57\% (6.9s) & 55\% (18.9s) & 50\% (1026s) & 13\% 
		(92s) & 13\% \bf (46s) \\ \hline
		MNIST 5$\times$200 & 50\% (6.5s) & 47\% (17s) & 46\% (930s) & 9\% (80s) & \bf 8\% (71s) \\ \hline
		MNIST 8$\times$100 & 63\% (7.2s) & 58\% (20s) & 58\% (1163s) & 21\% (102s) & \bf 15\% (61s) \\ \hline
		MNIST 8$\times$200 & 56\% (6.8s) & 55\% (18s) & 54\% (1083s) & 16\% (83s) & \bf 8\% (78s) \\ \hline
		MNIST 6$\times$500 & 53\% (6.4s) & 51\% (16s) & 50\% (1002s) & $-$ & 
		\bf 10\% (402s) \\ \hline
		CNN-B-Adv & 28\% (4.3s) & 22\% (8.7s) & 20\% (373s) & $-$ & \bf 11\% (417s) \\ \hline \hline
		ResNet & 0\% (2s) & 0\% (2s) & 0\% (2s) & $-$ & 0\% (2s) \\ \hline \hline
	\end{tabular}
	\caption{Undecided images ($\%$, {\em lower is better}) as computed by $\alpha,\beta$-CROWN, Refined $\beta$-CROWN and Hybrid MILP on 7 DNNs (average runtime per image). The 6 first DNNs are hard instances. The last DNN (ResNet) is an easy instance (trained using Wong to be easy to verify), provided for reference.
	}
	\label{table_hybrid}
	\end{table}

{\em Analysis}: overall, Hybrid MILP is very accurate, only leaving $8\%$-$15\%$ of images undecided, with runtime taking less than $500s$ in average per image, and even 10 times less on smaller DNNs. It can scale up to quite large hard DNNs, such as CNN-B-Adv with 2M parameters.

Compared with $\alpha,\beta$-CROWN: on easy instances, Hybrid MILP is virtually similar to $\alpha,\beta$-CROWN (e.g. even on the very large ResNet), since Hybrid MILP calls first $\alpha,\beta$-CROWN as it is very efficient for easy instances.

On hard instances (the 6 first DNNs tested), compared with $\alpha,\beta$-CROWN with a time-out of TO=2000s, Hybrid MILP is much more accurate, with a reduction of undecided images by $9\%-43\%$. It is also from 20x faster on smaller networks to similar time on the largest DNN. Compared with $\alpha,\beta$-CROWN with a time-out of TO=30s, the accuracy gap is even larger (e.g. $11\%$ for CNN-B-Adv, i.e. half the undecided images), although the average runtime is also obviously larger (solving hard instances takes longer than solving easy instances).

Last, compared with {\em Refined} $\beta$-CROWN, we can observe three patterns: on the shallowest DNNs (5$\times$100, 5$\times$200), Refined $\beta$-CROWN can run full MILP on almost all nodes, reaching almost the same accuracy than  Hybrid MILP, but with longer runtime (up to 2 times on 5$\times$100). As size of DNNs grows (8$\times$100, 8$\times$200), full MILP invoked by Refined $\beta$-CROWN can only be run on a fraction of the neurons, and the accuracy is not as good as Hybrid MILP, with $6\%-8\%$ more undecided images (that is double on 8$\times$200), while having longer runtime. Last but not least, Refined $\beta$-CROWN cannot scale to larger instances (6$\times$500, CNN-B-Adv), while Hybrid MILP can.

\subsection{Comparison with other Verifiers}

We voluntarily limited the comparison so far to $\alpha,\beta$-CROWN because it is one of the most efficient verifier to date, which allowed us to consider a spectrum of parameters to understand $\alpha,\beta$-CROWN scaling without too much clutter.

Interestingly, GCP-CROWN \cite{cutting} is slightly more accurate than $\alpha,\beta$-CROWN on the DNNs we tested, but necessitates IBM CPLEX solver, which is not available to us.
We provide in Table \ref{table_gcp} results from the Table 3 page 9 of \cite{cutting}, experimenting with different verifiers on the most interesting CIFAR CNN-B-Adv. Notice that results are not exactly comparable with ours because the testing images are not the same (ours has $62\%$ upper bound while \cite{cutting} has $65\%$ upper bound, so the image set in \cite{cutting} are slightly easier). There, GCP-CROWN is $2\%$ more accurate (with higher runtime) than $\alpha,\beta$-CROWN with 90s TO, so we can deduce that Hybrid MILP, which is $9\%$ more accurate than $\alpha,\beta$-CROWN with 2000s TO, is significantly more accurate than GCP-CROWN.


\begin{table}[b!]
	\centering
		\caption{Images verified by different verifiers ($\%$, higher is better) on CIFAR CNN-B-Adv. Results, from \cite{cutting}, are not fully comparable with Tables \ref{table_beta},\ref{table_hybrid}.}
	\label{table_gcp}
	\begin{tabular}{||c||c|c|c|c|c||}
		\hline
		Upper Bound & SDP-FO & PRIMA & refined $\beta$-CROWN & $\beta$-CROWN & GCP-CROWN \\  \hline
		65\% & 32.8\% & 38\% & 27\% & 46.5\% & 48.5\% \\
		$\epsilon = 2/255$ & $>25$h & $344$s & $361$s & $32$s &  $58$s  \\  \hline
	\end{tabular}
\end{table}

Results on other networks from other verifiers (PRIMA \cite{prima}, SDP-FO \cite{SDPFI}, etc) were already reported \cite{crown} on other tested DNNs, with unfavorable comparison vs $\alpha,\beta$-CROWN. Further, we reported accuracy of NNenum \cite{nnenum}, Marabou \cite{Marabou,Marabou2}, respectively 4th, 3rd of the last VNNcomp'24 \cite{VNNcomp24}, as well as full MILP \cite{MILP} in Table \ref{table_complete}, showing that these verifiers are not competitive on complex (even small) instances. Concerning MnBAB \cite{ferrari2022complete}, and it compares slightly unfavorably in time and accuracy towards $\alpha,\beta$-CROWN on CNN-B-Adv and {\em complex} MNIST DNNs at several time-out settings. Last, Pyrat \cite{pyrat} (2nd in the latest VNNComp) is not open source, which made running it impossible. 

%% file: proofsb.tex
\vspace{-0.6cm}

\section*{Appendix}

\section{Parameter settings}

\subsection*{Setting for Hybrid MILP}

Hybird MILP first call $\alpha,\beta$-CROWN with short time-out (TO), then call partial MILP on those inputs which was neither certified nor falsified by this run of $\alpha,\beta$-CROWN. We are using two settings of TO, for smaller DNNs we use TO$=10s$, and for the two larger ones, we use TO$=30s$.

Partial MILP uses 20 CPU-threads, while $\alpha,\beta$-CROWN uses massively parallel ($>$4096 threads) GPU,

The setting for partial MILP for fully-connected DNNs is about how many neurons need to be opened (once set, the selection is automatic). The runtime depending crucially upon the number of open ReLU neurons, we set it quite tightly, only allowing few neuron deviation to accommodate to a particularly accurate/inaccurate bound computation (measured by the weight of the remaining SAS score). As complexity increases with the layer considered, as the size of the MILP model grows, we lower this number with the depth, only committing to an intermediate number for the output neuron (the number of output neurons  is smaller than hidden layer, and this is the most important computation). We experimentally set this number so that each computing the bounds in each hidden layer takes around the same time. Remember that in layer 1, partial MILP is not necessary and propagating bounds using interval arithmetic is already exact. We open [48,48] to compute bounds for hidden layer 2, [21,24] for layer 3, [11,14] for layer 4, [6,9] for layer 5, [3,6] for layer 6, [2,5] for layer 7, [1,4] for hidden layer 8 (if any), and we open [14,17] for the output layer.
 The exact number of open nodes in the range [a,a+3] is decided automatically for each neuron being computed : ReLUs are ranked according to their value by SAS, and the a top ReLUs are open. Then, ReLUs ranked a+1,a+2, a+3 are opened if their SAS value is larger than a small threshold. We set the threshold at 0.01. It should be seen as a way to save runtime when SAS knows that the next node by ranking (a+i) will not impact accuracy much (thanks to the upper bound from Proposition \ref{prop2}).

\begin{table}[b!]
	\centering
		\caption{Settings of Hybrid MILP for the different {\em hard} instances}
	\begin{tabular}{||l||c|c||}
		\hline \hline
		Network & TO for $\alpha,\beta$-CROWN  & Minimum number of Open neurons  \\ 		  
		\hline
		MNIST $5 \times 100$ & 10s  & 48,21,11,6,14  \\ \hline
		MNIST $5 \times 200$ & 10s & 48,21,11,6,14  \\ \hline
		MNIST $8 \times 100$ & 10s  & 48,21,11,6,3,2,1,14  \\ \hline
		MNIST $8 \times 200$ & 10s & 48,21,11,6,3,2,1,14  \\ \hline
		MNIST $6 \times 500$ & 30s & 48,21,11,6,3,14 \\ \hline
		CIFAR CNN-B-Adv & 30s & 200, 0, 45 \\ \hline \hline
	\end{tabular}
	\label{table20}
	\end{table}

For convolutional CNNs, the strategy is adapted, as there is much more neurons, but in a shallower architecture and not fully connected. 
The second layer is computed accurately, opening 200 neurons, which is manageable as there is only one ReLU layer to consider, and accuracy here is crucial.
We do not open any nodes in the third layer (the first fully connected layer) if the output layer is the next one (which is the case for CNN-B-Adv), and instead rely on the choice of important nodes for the output layer. Otherwise, we open 20 neurons.
In the output layer, we open at least 45 neurons (there is less output neurons than nodes in the previous layer), and enlarge the number of open neurons (up to 300) till we find an upper bound, that is a best current MILP solution, of around +0.1 (this 0.1 was experimentally set as target, a good balance between accuracy and efficiency), and compute a guaranteed lower bound (the goal is to guarantee the bound is $>0$).

Table \ref{table20} sums up the TO and the minimum numbers of ReLU opened.

Last, for Gurobi, we use a custom MIP-Gap (from $0.001$ to $0.1$) and time-out parameters, depending on the seen improvement and the possibility to make a node stable. This is low level implementation details that will be available in the code once the paper is accepted.

Notice that a different balance between accuracy and runtime could be set. For instance, we set up the numbers of open neurons to have similar runtime as Refined $\beta$-CROWN for the first 4 DNNs ($50s-100s$). We could easily target better accuracy (e.g. for $8 \times 100$ with a relatively high $15\%$ undecided images) by increasing the number of open neurons, with a trade-off on runtime (current runtime is at $61s$).
By comparison, the sweet spot for $\alpha,\beta$-CROWN seems to be around TO$=30s$, enlarging the time-out having very little impact on accuracy but large impact on runtime
(Table \ref{table_beta}).

\subsection*{Setting for $\alpha,\beta$-CROWN}

The networks were already tested by $\alpha,\beta$-CROWN \cite{crown}. We thus simply reused the parameter files from \href{https://github.com/Verified-Intelligence/alpha-beta-CROWN/blob/main/complete_verifier/exp_configs/beta_crown/}{their Github}, 
except for time-out which we explicitly mention:
e.g. for CNN-B-Adv: "solver: batch size: 512 beta-crown: iteration: 20"; and
for MNIST 5x100: "solver: batch size: 1024 beta-crown: iteration: 20".


\section{Additional experiments including Ablation studies}

First, we provide Fig.~\ref{fig5}, similar to 
Fig.~\ref{fig_table3} but on a different image, namely 37.
It display similar pattern as image 85, with an even larger difference between SAS and GS (around 3x more binary variables are necessary for GS to match the accuracy of SAS).

\begin{figure}[h!]
	\hspace*{-1cm}
	\includegraphics[scale=0.55]{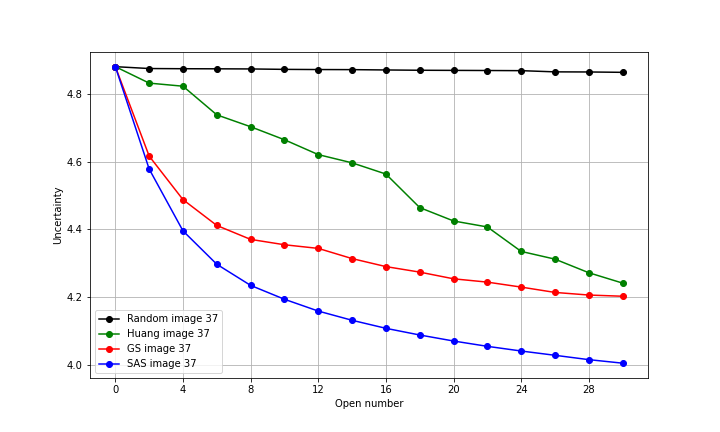}
	\caption{Comparison of different methods with image 37.}
	\label{fig5}
\end{figure}

%

\begin{table}
		\caption{Distance to verify and runtime (lower is better) for different output nodes for image 85. Here, negative distance means the node is verified.}
	\begin{tabular}{|c||>{\centering\arraybackslash}p{6ex}|>{\centering\arraybackslash}p{6ex}|>{\centering\arraybackslash}p{6ex}|>{\centering\arraybackslash}p{6ex}|>{\centering\arraybackslash}p{6ex}|>{\centering\arraybackslash}p{6ex}|>{\centering\arraybackslash}p{6ex}|>{\centering\arraybackslash}p{6ex}|>{\centering\arraybackslash}p{6ex}||>{\centering\arraybackslash}p{6ex}||>{\centering\arraybackslash}p{6ex}|}
		\hline
		Output Node & 0 &  2 & 3 & 4 & 5 & 6 & 7 & 8 & 9 &Avg& Var \\
		\hline
		{\sf SAS} distance & -1.47 & -4.56 & -4.48 & -3.46 & -5.6 & -3.06 & -4.17 & 0.639 & -2.33 &-3.16& \\
		Time (s) & 66.9 & 54.8 & 69.9 & 88.8 & 94.8 & 63.3 & 73.4 & 56.7 & 64.6 &70.3& 184\\
		\hline
		{\sf GS} distance & -1.38 & -4.54 & -4.45 & -3.48 & -5.7 & -2.98 & -4.08 & 0.726 & -2.28 &-3.13&\\
		Time (s) & 71.6 & 84.9 & 71.2 & 113 & 93.7 & 68.5 & 88.6 & 46.5 & 69.4 &78.6&359\\
		\hline
	\end{tabular}
		\label{SAS_GS_similarTO}
	\end{table}
	
Then, we provide Table \ref{SAS_GS_similarTO} to show the runtime variability of {\sf GS} and {\sf SAS}. The number of opens nodes for {\sf SAS} is 14 and the number of open nodes for {\sf GS} is 30, to account for the lack of accuracy of SAS, making the average distance to verify similar between both ($-3.16$ vs $-3.13$). The runtime variability between different output nodes is twice as high for {\sf GS} vs {\sf SAS}, because the number of irrelevant nodes chosen by {\sf GS} is inpredictable. If this number is high, then the runtime can be particularly short (node 8), with also a worse accuracy; while if it is low (e.g. 4), then the runtime is particularly long. This does not necesarily translates into better accuracy, e.g. node 2 where the runtime is much longer for {\sf GS} vs {\sf SAS}, while the accuracy is also lower. {\sf SAS} displays much less variability in time between different output nodes as most nodes are relevant.

Further, we consider ablation studies to understand how each feature enables the efficiency of pMILP.

\subsection*{Time scaling with open nodes}	

First, we explore the time scaling with different number of open nodes, for SAS using nodes in the last two layers (Layer 1 and 2) wrt nodes of layer 3 of $5\times 100$ on image 59 of MNIST, presented in Table \ref{table14} and Fig. \ref{fig3}.

\begin{table}[t!]
	 \caption{Time and uncertainty scaling of pMILP with number of nodes.}
		\centering
		\hspace*{4ex}
		\begin{subtable}[b]{0.45\textwidth}
		\begin{tabular}{|c|c|c|}
		\hline
		$|X|$ & Time & Uncertainty\\ 
		\hline	0 & 2.6 & 1.760946128\\
		\hline	1 & 7.3 & 1.702986873\\
		\hline	2 & 11.1 & 1.65469034\\
		\hline	3 & 16.3 & 1.612137282\\
		\hline	4 & 15.5 & 1.571001109\\
		\hline	5 & 15.7 & 1.531925404\\
		\hline	6 & 15.8 & 1.49535638\\
		\hline	7 & 16.4 & 1.46189314\\
		\hline	8 &  15.8 & 1.4299535\\
		\hline	9 &  17.2 & 1.4006364\\
		\hline	10 & 22.5 & 1.3711203\\
		\hline	11 & 27.2 & 1.3438245\\
		\hline	12 & 21.6 & 1.3183356\\
		\hline	13 & 28.7 & 1.2938690\\
		\hline	14 & 29.6 & 1.2690507\\
		\hline	15 & 24.5 & 1.2475106\\
		\hline
	  \end{tabular}
	\end{subtable}
	\hfill
	\begin{subtable}[b]{0.45\textwidth}
		\begin{tabular}{|c|c|c|}
			\hline
			$|X|$ & Time & Uncertainty\\ 
		\hline	16 & 31.9 & 1.2243065\\
		\hline	17 & 28.6 & 1.2031791\\
		\hline	18 & 30.4 & 1.1839474\\
		\hline	19 & 34.0 & 1.1644653\\
		\hline	20 & 42.1 & 1.1456181\\
		\hline	21 & 47.6 & 1.1261252\\
		\hline	22 & 62.7 & 1.1089745\\
		\hline	23 & 70.0 & 1.0931242\\
		\hline	24 & 70.8 & 1.0773088\\
		\hline	25 & 139.9 & 1.060928\\
		\hline	26 & 154.2 & 1.045715\\
		\hline	27 & 213.1 & 1.030605 \\
		\hline	28 & 211.3 & 1.016058\\
		\hline	29 & 373.1 & 1.001374\\
		\hline max=116 & 3300 & 0.895\\ 
		\hline		
	  \end{tabular}
     \end{subtable}
    	\label{table14}
\end{table}

\begin{figure}[h!]
	\vspace*{-0.8cm}
	\includegraphics[scale=0.6]{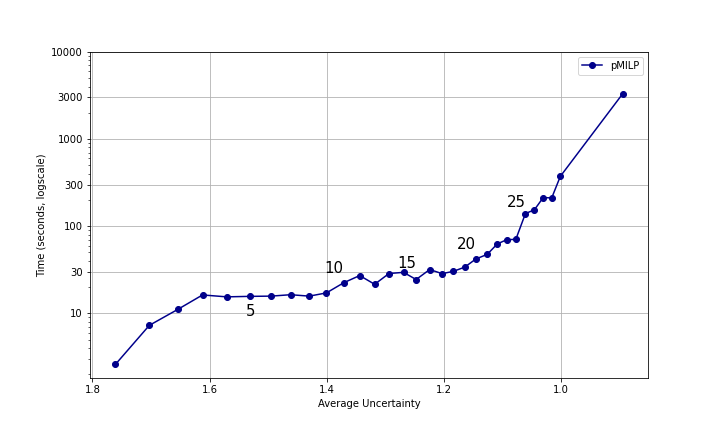}.
	\caption{Time and uncertainty scaling of pMILP with number of nodes.
	Time is using logscale.}
	\label{fig3}
\end{figure}

The exponential complexity with the number of nodes can be seen on Figure \ref{fig3}, where time is represented using logarithmic scale. The flat area in the middle is Gurobi having good heuristic to avoid considering all $2^K$ cases when $K<21$ is not too large, but not working so well for $K>25$. Notice that when certifying, pMILP uses $|X| \in$ 21-24, which is a good trade off between time and accuracy.

\subsubsection*{restricting number of open nodes (pMILP) vs setting time-outs (full MILP)}	

Running full MILP till a small MIP-Gap (typically 0.001) is reached is extremely time inefficient.

Instead, the standard strategy is to set a reasonable time-out and use whatever bound has been generated. We compare this standard strategy with the pMILP strategy of setting a priori a number of open nodes.

\begin{figure}[h!]\hspace*{-0.8cm}
	\includegraphics[scale=0.6]{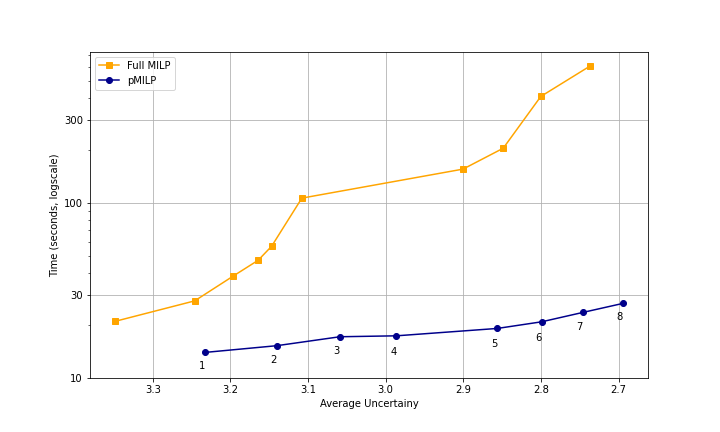}.
	\caption{Comparison of uncertainty at layer 7 for full MILP with different time-outs vs pMILP with different number of open nodes. Time is using logscale.}
	\label{fig4}
\end{figure}

\begin{table}[h!]
	\caption{Comparison of bounding the number of nodes for pMILP and 
		using different time outs for full MILP. In both settings, lower and upper bounds of previous layers are the same (computed by pMILP).}
	\centering
	\hspace*{4ex}
\begin{subtable}[b]{0.45\textwidth}
	\centering
		\begin{tabular}{|c|c|c|}
	\hline
		$|X|$ & Time & Uncertainty\\ 
	\hline1 &	14 & 3.233021901\\
\hline	2 & 15.2 & 3.140309921\\
\hline	3 & 17.21 & 3.059083103\\
\hline 4 &	17.4 & 2.986166762\\
\hline	5 &19.2 & 2.856229765\\
\hline	6 &20.9 & 2.799248232\\
\hline	7 &23.7 & 2.746167245\\
\hline	8 &26.6 & 2.69485246\\	
	\hline
	\end{tabular}
	\caption{pMILP}
\end{subtable}
\hfill
\begin{subtable}[b]{0.45\textwidth}
	\centering
		\begin{tabular}{|c|c|}
		\hline
		Time & Uncertainty\\ 
		\hline	21.1 & 3.348236261\\
		\hline	27.6 & 3.24604282\\
		\hline	38.2 & 3.196640184\\
		\hline	47.1 & 3.164298172\\
		\hline	56.7 & 3.146913614\\
		\hline	106.7 & 3.108035223\\
		\hline	156.3 & 2.900438725\\
		\hline	205.8 & 2.848648426\\	
		\hline	406.7 & 2.800268264 \\	
		\hline	606.1 & 2.737064255\\	
		\hline
	\end{tabular}
		\caption{full MILP}
\end{subtable}
	\label{table12}
	\end{table}

pMILP obtains 2.8 accuracy in $<21$ seconds (with 7 open nodes), while full MILP needs 400 seconds to obtain it, a 19x speed up. For 2.7 accuracy, the speedup is $>>$ 22.

Figure \ref{fig4} shows that choosing nodes is much more efficient for time/accuracy trade-off than setting time outs and use full MILP. And this is for the smallest DNN we considered (500 hidden neurons, far from the biggest 20k neuron DNN we experimented with)

\section{Comparison with other DNN verifiers}

In the following, we provide results comparing $\alpha,\beta$-CROWN to other verifiers, to justify our use of $\alpha,\beta$-CROWN as state of the art for efficient verifiers as main source of comparison to Hybrid MILP for hard DNN instance.

\subsection*{Comparison $\alpha,\beta$-CROWN vs PRIMA}

\begin{table}[h!]
		\caption{Undecided images ($\%$, {\em lower is better}), as computed by $\alpha,\beta$-CROWN, Refined $\beta$-CROWN, and PRIMA, as reported in \cite{crown}, except for $6 \times 500$ that we run ourselves. N/A means that \cite{crown} did not report the numbers, while $-$ means that Refined $\beta$-CROWN cannot be run on these DNNs.}
	\centering
	\begin{tabular}{||l||c|c||c||}
		\hline \hline
		Network & $\alpha,\beta$-CROWN & Refined $\beta$-CROWN & PRIMA \\ 		  
		\hline
		MNIST $5 \times 100$ & N/A  & 14.3\% (102s) & 33.2\% (159s)\\ \hline
		MNIST $5 \times 200$ & N/A & 13.7\% (86s) & 21.1\% (224s) \\ \hline
		MNIST $8 \times 100$ & N/A  & 20.0\% (103s) & 39.2\% (301s)   \\ \hline
		MNIST $8 \times 200$ & N/A & 17.6\% (95s) & 28.7\% (395s)  \\ \hline
		MNIST $6 \times 500$ & 51\% (16s) & $-$ & 64\% (117s) \\ \hline
		CIFAR CNN-B-Adv & 18.5\% (32s) & $-$ & 27\% (344s)\\ \hline \hline
		CIFAR ResNet & 0\% (2s) & $-$ & 0\% (2s) \\ \hline \hline
	\end{tabular}
	\label{table9}
	\begin{tablenotes}
		\footnotesize
		\item Most data is directly from \cite{crown}. N/A means no data either in \cite{crown} or by our running.
		\item  $^*$ The data in this row is from our own running on first 100 images of the MNIST dataset.
		\item  $^{**}$ The data is from \cite{crown} on first 200 images of the CIFAR10 dataset.
	\end{tablenotes}
	\end{table}

	PRIMA \cite{prima} is a major verifier in the ERAN toolkit. In Table \ref{table9}, we report the comparison between PRIMA and $\alpha,\beta$-CROWN, mainly from \cite{crown}. The setting is mainly similar from ours, but numbers are not perfectly comparable as the images tested are not  exactly the same (1000 first or 200 first images for CNN-B-Adv), vs 100 first in Tables \ref{table_hybrid}, \ref{table_beta}. Also, time-out settings and hardware are slightly different. The overall picture is anyway the same.

Analysis: On the 4 smallest MNIST networks, PRIMA uses a refined path comparable with Refined $\beta$-CROWN. However, it is slower and less accurate than Refined $\beta$-CROWN.
On larger {\em hard} networks, PRIMA has also more undecided images than $\alpha,\beta$-CROWN, while the runtime is $>5$ times larger.
Hence, Hybrid MILP is more accurate than PRIMA with similar runtime or faster.

Notice that kPoly \cite{kpoly}, OptC2V \cite{optC2V}, SDP-FO \cite{SDPFI} numbers were also reported in \cite{crown} on these networks, with even more unfavorable results.

\subsection*{Comparison $\alpha,\beta$-CROWN vs MN-BaB}

MN-BaB \cite{ferrari2022complete} is an improvement built over PRIMA, using a similar Branch and Bound technique as used in $\alpha,\beta$-CROWN. Results in \cite{ferrari2022complete}
are close to those of $\alpha,\beta$-CROWN. However, none of the {\em hard} networks from \cite{crown} that we consider have been tested. We thus tested three representative {\em hard} DNNs (first 100 images) to understand how MN-BaB fairs on such hard instances, and report the numbers in Table \ref{table10}. Results are directly comparable with Table \ref{table_hybrid}.

\begin{table}[h!]
	\centering
		\caption{Undecided images ($\%$, {\em lower is better}), as computed by $\alpha,\beta$-CROWN, and MN-BaB}
	\begin{tabular}{||l||c|c||c|c||}
		\hline \hline
		 & $\alpha,\beta$-CROWN & $\alpha,\beta$-CROWN & MN-BaB & MN-BaB \\ 
		 Network & TO=30s & TO=2000s &  TO=30s & TO=2000s \\ 
		\hline
		MNIST $5 \times 100$ & 55\% (19s) & 50\%(1026s) & 60\% (19s) & 50\% (1027s) \\ \hline
		MNIST $6 \times 500$ & 51\% (16s) & 50\% (1002s) & 58\% (18s) & 55\% (1036s) \\ \hline
		CIFAR CNN-B-Adv & 22\% (8.7s) & 20\% (373s) & 43\% (14s) & 24\% (576s) \\ \hline 
	\end{tabular}
	\label{table10}
\end{table}

Analysis: results reveal that MN-BaB is slightly slower and slightly less accurate than $\alpha,\beta$-CROWN. Notice the specially high number of undecided images for CNN-B-Adv with TO=30s, probably meaning that 30s is too small for MN-BaB on this large DNN.
Hence, Hybrid MILP is more accurate than MN-BaB with similar runtime or faster.

	\subsection*{Comparison $\alpha,\beta$-CROWN vs NNenum}

NNenum \cite{nnenum} is a complete verifier with good performance according to VNNcomp.
It was the only complete verifier tested in Table \ref{table_complete} to verify more images than $\alpha,\beta$-CROWN. The experiments section in \cite{nnenum} does not report
the {\em hard} DNNs we are considering. We tried to experiment it on the same MNIST 
$6 \times 500$ and CIFAR CNN-B-Adv as we did in Table \ref{table10} for MN-BaB. Unfortunately, on $6 \times 500$, buffer overflow were reported.
We report in Table \ref{table11} experiments with the same 2000s Time-out (it was $10 000s$ in Table \ref{table_complete})  for a fair comparison with $\alpha,\beta$-CROWN, on both 
MNIST $5 \times 100$ and CIFAR CNN-B-Adv. 
On MNIST $5 \times 100$, NNenum is slightly more accurate than $\alpha,\beta$-CROWN, but far from the accuracy Hybrid MILP.
On CIFAR CNN-B-Adv, NNenum was much less accurate than $\alpha,\beta$-CROWN, and thus of Hybrid MILP. In both test, the runtime of NNenum was also much longer than for Hybrid MILP.

\begin{table}[h!]
	\centering
	\begin{tabular}{||l||c||c||c||c||}
		\hline \hline
		 & $\alpha,\beta$-CROWN & NNenum & Hybrid\\ 
		 Network & TO=2000s &  TO=2000s & MILP\\ 
		\hline
		MNIST $5 \times 100$ & 50\%(1026s) & 44\% (1046s) & \bf 13\% (46s)\\ \hline
		CIFAR CNN-B-Adv & 20\% (373s) & 40\% (1020s) & \bf 11\% (417s)\\ \hline 
	\end{tabular}
	\caption{Undecided images ($\%$, {\em lower is better}), as computed by $\alpha,\beta$-CROWN and NNenum with 2000s time-out, and Hybrid MILP}.
	\label{table11}
\end{table}

\section{Average vs max time per pMILP call}

We provide in Table \ref{table112} the average as well as maximum time to perform $\MILP_X$ calls as called by pMILP, on a given input: image 3 for MNIST, and image 76 for CIFAR10. 
For 6x500, we provide results for two different $\varepsilon$.


\begin{table}[h!]
	\centering
	\begin{tabular}{||l|c|c||}
		\hline
		Network & average time & maximum time \\ \hline
		MNIST 5$\times$100 & 0.41s & 1.87 \\
		$\epsilon = 0.026$ &  & \\  \hline
		MNIST 5$\times$200 &  0.75s & 5.31s \\ 
		$\epsilon = 0.015$ & & \\  \hline
		MNIST 8$\times$100 & 0.39s & 1.41s \\
		$\epsilon = 0.026$ & &  \\  \hline
		MNIST 8$\times$200 & 0.49s & 1.63s \\ 
		$\epsilon = 0.015$ & & \\  \hline
		MNIST 6$\times$500 & &   \\ 
		$\epsilon = 0.035$ & 1.4s & 3.5s \\ 
		$\epsilon = 0.1$ & 44.6s & 310s \\  \hline 
		CIFAR CNN-B-Adv &  & \\
		$\epsilon = 2/255$& 1s & 609s \\ \hline \hline
	\end{tabular}
	\caption{average and maximum time per $\MILP_X$ calls for image 3 (MNIST) and image 76 (CIFAR10).}
	\label{table112}
\end{table}

Notice that DNN 6$\times$ 500 and $\epsilon=0.1$ is a very hard instance as being very close to the falsification $\epsilon \approx 0.11$. This is thus not representative of the average case. Also,  on this image 3, pMILP succeeds to verify $\epsilon= 1.054$, while $\alpha,\beta$-CROWN can only certify $\epsilon = 0.0467$ within the 10 000s Time-out.

For CNN-B-Adv, the very long maximum time for a MILP call is an outlier: it happens only for one output layer, for which the number $K$ of open nodes is particularly large (around 200 out of 20000 neurons) to certify this hard image 76. Indeed, the average time is at $1s$. Notice that this does not lead to a runtime of 20.000s, as 20 threads are used by pMILP 
in parallel (similar to competing solutions, except $\alpha,\beta$-CROWN which uses $>4096$ GPU cores).